\documentclass{ecai} 

\usepackage{latexsym}
\usepackage{amssymb}
\usepackage{amsmath}
\usepackage{amsthm}
\usepackage{booktabs}
\usepackage{enumitem}
\usepackage{graphicx}
\usepackage{color}
\usepackage{algorithm}
\usepackage{algorithmic}
\usepackage{multirow}
\usepackage{makecell}
\usepackage{array}
\usepackage{tabularray}
\usepackage{import}
\usepackage{dsfont}
\usepackage{mathtools}

\date{}

\newtheorem{theorem}{Theorem}

\newtheorem{definition}{Definition}
\newtheorem{proposition}{Proposition}
\theoremstyle{remark}
\newtheorem*{remark}{Remark}

\newcommand\shortrightarrow{\scalebox{0.7}[1]{$\rightarrow$}}

\newcommand{\BibTeX}{B\kern-.05em{\sc i\kern-.025em b}\kern-.08em\TeX}

\begin{document}


\begin{frontmatter}




\title{Cautious Calibration in Binary Classification}



\author[A]{\fnms{Mari-Liis}~\snm{Allikivi}\thanks{Corresponding Author. Email: mari-liis.allikivi@ut.ee}}
\author[A]{\fnms{Joonas}~\snm{Järve}}
\author[A]{\fnms{Meelis}~\snm{Kull}} 


\address[A]{Institute of Computer Science, University of Tartu, Estonia}


\begin{abstract}
Being cautious is crucial for enhancing the trustworthiness of machine learning systems integrated into decision-making pipelines. Although calibrated probabilities help in optimal decision-making, perfect calibration remains unattainable, leading to estimates that fluctuate between under- and overconfidence. This becomes a critical issue in high-risk scenarios, where even occasional overestimation can lead to extreme expected costs. In these scenarios, it is important for each predicted probability to lean towards underconfidence, rather than just achieving an average balance. In this study, we introduce the novel concept of cautious calibration in binary classification. This approach aims to produce probability estimates that are intentionally underconfident for each predicted probability. We highlight the importance of this approach in a high-risk scenario and propose a theoretically grounded method for learning cautious calibration maps. Through experiments, we explore and compare our method to various approaches, including methods originally not devised for cautious calibration but applicable in this context. We show that our approach is the most consistent in providing cautious estimates. Our work establishes a strong baseline for further developments in this novel framework.

\end{abstract}

\end{frontmatter}


\section{Introduction}

Classifier calibration is used to improve probabilistic predictions of classification models. For example, in binary classification, if a set of instances were assigned a predicted positive class probability of 0.8, then 80\% of these instances should actually be positive. This is an appealing property both intuitively for humans to interpret the model's predictions and for optimal decision making. 

A key application of calibrated probability estimates lies in their usefulness in mediating between the model training stage and the downstream application. Since exact knowledge of costs and other variables in downstream tasks is often unavailable or subject to change, calibrated models are invaluable in these practical scenarios by being adaptable to these fluctuations \citep{ovadia_can_2019}. This makes the further development of different post hoc calibration methods like isotonic calibration \citep{zadrozny2002transforming}, logistic calibration \citep{platt1999probabilistic}, beta calibration \citep{kull2017beta} and other approaches for training calibrated models like temperature scaling, matrix and vector scaling \citep{guo2017calibration} or Dirichlet calibration \citep{kull2019beyond} an active area of research. 

\begin{figure}[h]
\centering
\includegraphics[width=8.5cm]{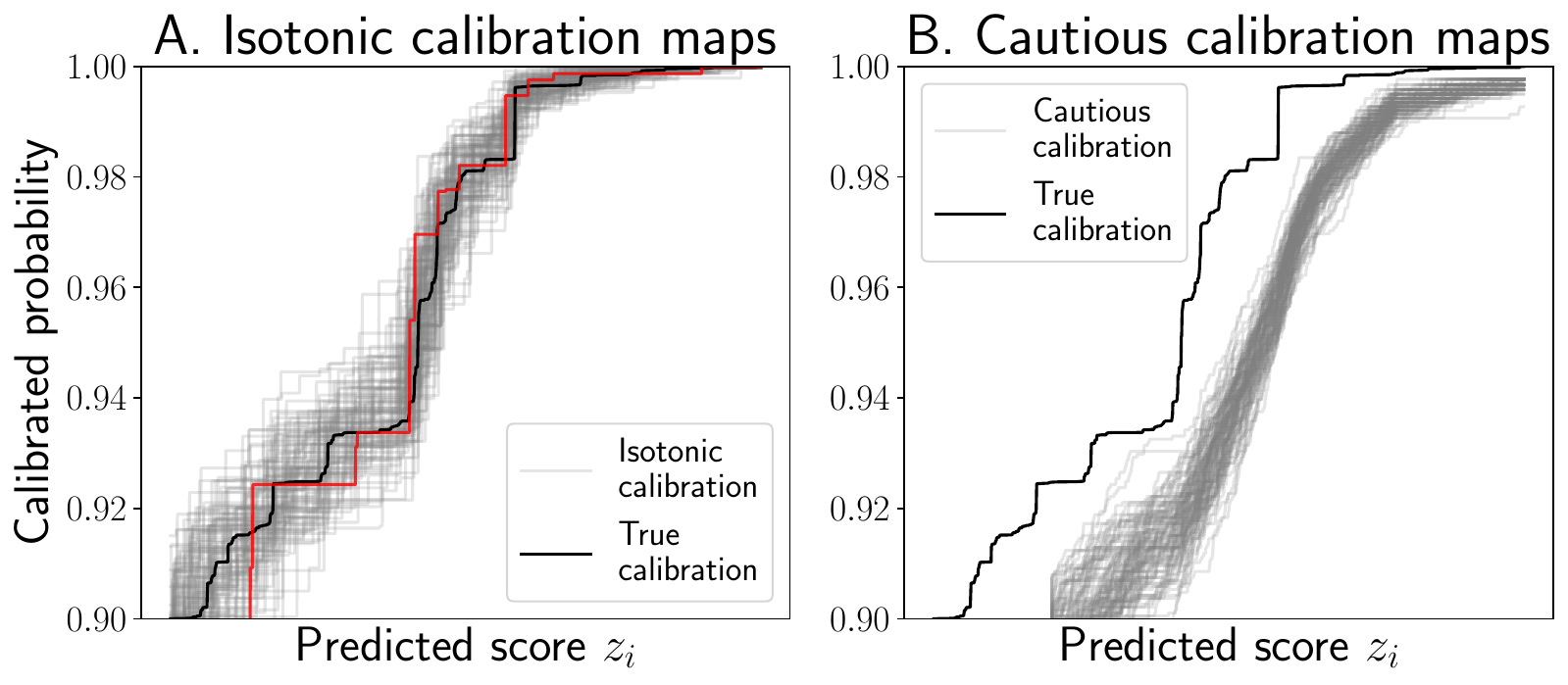}
\caption{Comparison of 100 learned calibration maps (grey) based on one simulated true calibration map (black). A. Maps learned by isotonic calibration, with an arbitrary one highlighted in red. B. Maps learned with a cautious calibration approach (HTLB+CP).}
\vspace{18pt}
\label{fig:intro_figure}
\end{figure}

However, a universal challenge with all calibration methods is their inability to reach perfect calibration. These methods aim to learn the true calibrated values, which reflect the actual unknown conditional probabilities $Pr(Y=1|z_i)$ of observing class 1 for a predicted score $z_i$, i.e. the confidence predicted by a model. This ideal is not achievable due to having finite amounts of data, but also since methods have limitations due to their parametric nature (e.g., logistic, beta calibration) or biases like a tendency towards overconfidence (e.g., isotonic regression) \citep{allikivi2019non} or other reasons that cause under- and overconfidence \citep{bai2021don}. Even the most effective methods still estimate the true probability slightly inaccurately, with the direction of this deviation remaining unknown. This issue is illustrated in Figure \ref{fig:intro_figure}A. Suppose that we know the true calibration map\footnote{We generate synthetic true calibration maps, as the true calibration maps are always unknown for real data.} (black) and sample 100 calibration sets from it using the approach described in Section \ref{sec:exp_setup}. Isotonic calibration has been applied to all calibration sets and 100 estimated maps have been learned (grey). The plot demonstrates how isotonic calibration, one of the classical and widely used calibration methods, is on average quite precise in its estimation of the true calibration map, but individual estimations fluctuate between over- and underconfidence.

In situations where costs of different errors vary significantly, overconfidence can lead to excessively high costs while underconfidence results in slightly less optimal costs, as will be demonstrated in Section \ref{sec:optimalrisklevelselection}. As decision-making in case of imbalanced costs happens in the high-probability region, we focus only on probabilities 0.9 and higher, as we need high certainty to make risky decisions. From Figure \ref{fig:intro_figure}A we can see that while calibration methods can be precise on average, one single learned map (red) can be overconfident for some scores and underconfident for others. As overconfidence leads to worse outcomes in our problem setting, we want to avoid it for every individual map at every score. Behind every score there is a group of instances in our distribution that will be mapped to that prediction, from now on referred to as a \emph{score group}. In order to avoid bad outcomes for all score groups, we should expect every one of them to perform well, instead of only the model's average outcome being good.

Our research introduces a novel concept called \emph{cautious calibration} that, by design, addresses this problem. The aim is to learn calibration maps that always stay on the underconfident side, providing lower bounds for the true calibrated values. This is illustrated in Figure \ref{fig:intro_figure}B. The paper starts with an example scenario of optimal risk level selection where cautiousness is needed to ensure that there are no overly negative outcomes for any score group. This is followed by an overview of existing methods that could be used for learning cautious calibration maps. Next, we propose our own, more reliable and theoretically principled method for cautious calibration. Finally, we show that our work provides a strong baseline for cautious calibration and demonstrate its critical relevance in our example scenario.

\section{Example Scenario: Optimal Risk Level Selection}
\label{sec:optimalrisklevelselection}

The usefulness of cautious calibration can be demonstrated in a setting where each prediction leads to a decision with potentially different costs. Before defining this setting formally, let's build intuition by following a simplified scenario in the self-driving domain. Imagine a machine learning model that is trained to identify images of clear roads (class 1) and roads with obstacles (class 0). Based on the predicted probability (probability for the road to be clear), one has to select the speed of the car, or in other terms, select the risk level. The more certain we are, the higher speed we are willing to choose. However, the outcomes are dependant on the chosen speed. If we choose a high speed, then we will gain in faster arrival times when observing class 1, but lose greatly in case observing class 0 and causing an accident. This means that we should only select high risks in case of very high predicted probabilities, meaning that the chances of mistaking are extremely small. In case of lower speed, it will take longer to reach the destination, but accidents are less costly.

We call this setting \emph{optimal risk level selection} and formalize it in the following way. We have a probability prediction function $\hat{c}: \mathcal{X} \rightarrow [0,1]$, that, based on the input, predicts the estimates of the true calibrated probability, similarly to our model that predicted the probability for the road to be clear. We have an outcome function $o: (\mathcal{Y}, [0,\infty)) \rightarrow \mathbb{R}$, that calculates the outcome based on a class label $y$ and a risk level $\xi \in [0, \infty)$. This would equate to calculating the gain or cost \footnote{We will use the term \emph{outcome} as we have both gains (for positive outcomes) and costs (for negative outcomes). Higher outcome is better.} of selecting a speed and then observing either clear road or road with an obstacle, which in this toy example is obviously difficult to define. We have selected the following outcome function to illustrate our ideas:
\[
o(y, \xi ) = 
\begin{cases}
    \xi  & \text{if } y = 1, \\
    -\xi ^{l} & \text{if } y = 0,
\end{cases}
\]
where $l \in [1, \infty)$ is a fixed constant describing the imbalance of costs. This function is chosen for its simplicity and for how it captures the idea of small positive gains for seeing a positive class versus large negative costs for seeing a negative class, worsening greatly with the increase of the risk level.

The final missing part is to decide how to choose a risk level $\xi$  based on a true calibrated probability $c \in [0, 1)$ (we omit the trivial case of $c=1$). We do it by choosing the risk level that gives the highest expected outcome, assuming the data points are drawn from $Y \sim Bernoulli(c)$. This can be thought of as optimizing the outcome for one score group. We can define it formally as:
\begin{equation}
r_{opt}(c) = \arg\max_{\xi}\mathbb{E}_{Y \sim Bern(c)}[o(Y, \xi)] = \left(\frac{c}{l(1-c)}\right)^{\frac{1}{l-1}}  \label{argmaxrisk}
\end{equation}
where the derivation of the optimum is given in Supplementary \ref{proofs}.

Since in practice we don't have access to the true calibrated probability $c$, we use the estimation $\hat{c}$ to select the optimal risk level. However, as $c \neq \hat{c}$ in practice, we will not be able to choose the truly optimal risk level for the true underlying data generation process $Y \sim Bernoulli(c)$. This means that our chosen risk level $\xi_{\hat{c}} = r_{opt}(\hat{c})$ will result in an expected outcome $\mathbb{E}_{Y \sim Bern(c)}[o(Y, \xi_{\hat{c}})]$, which will be worse compared to the optimal one $\mathbb{E}_{Y \sim Bern(c)}[o(Y, \xi_c)]$.

Figure \ref{fig:outcome_example} illustrates the difference in how much the expected outcome worsens when the calibration estimate is overconfident vs underconfident. We create a hypothetical example where for each true probability $c$ we have one overconfident estimation $\hat{c}_{+} = c + 0.01$ and one underconfident estimation $\hat{c}_{-} = c - 0.01$. Then, for a range of different $c$'s represented on the x-axis, we find the true optimal risk level $\xi_{c} = r_{opt}(c)$ and the corresponding expected outcome $\mathbb{E}_{Y \sim Bern(c)}[o(Y, \xi_c)]$, the latter illustrated by the solid black line. The dotted line represents the expected outcome when the risk level is chosen with the underconfident estimation $\hat{c}_{-}$. We can see that the expected outcome is smaller than optimal, especially for high probabilities, but at the same time it stays on the positive side. For the analogous situation with the overconfident estimation $\hat{c}_{+}$ shown as the dashed line, we see a different trend, where in case of overestimation of high probabilities, our chosen risk level is too large to an extent where it results in extremely negative expected costs.

\begin{figure}[t]
    \label{fig:outcome_example}
  \centering
  \includegraphics[width=8.5cm]{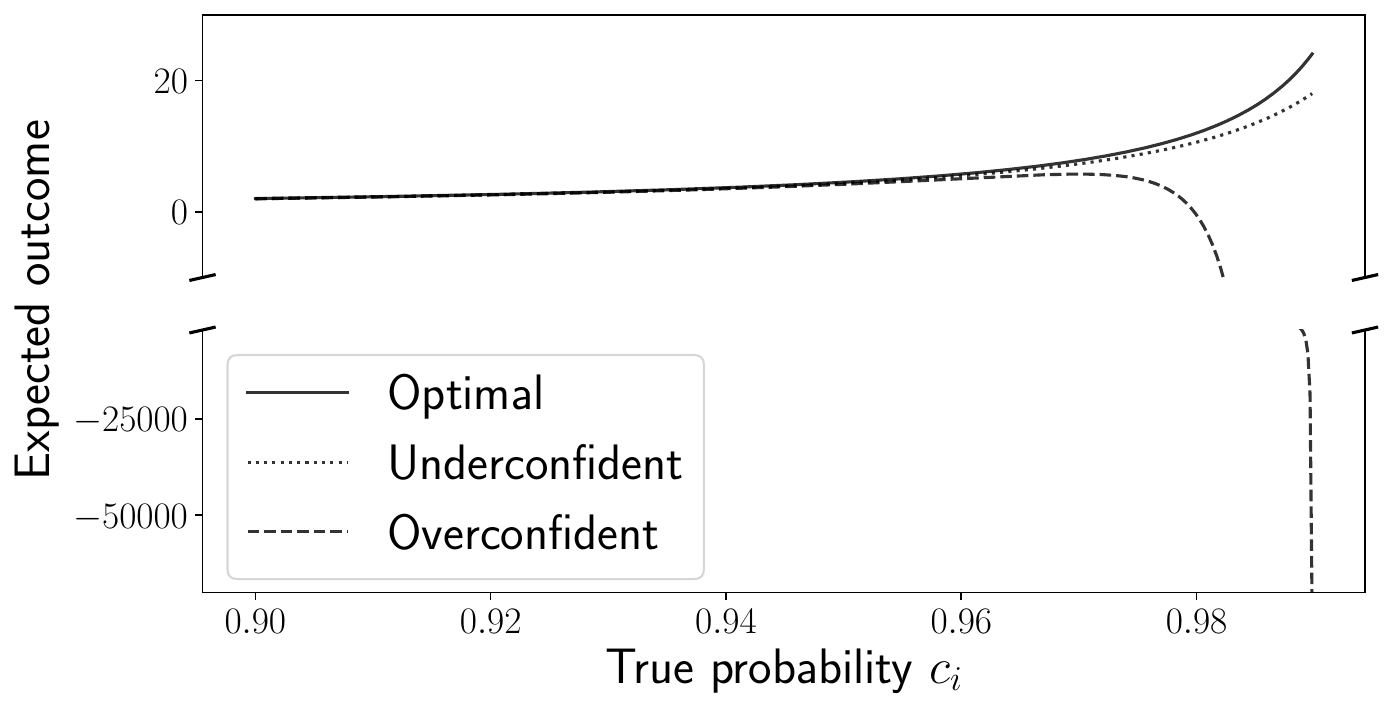}
  \caption{Expected outcomes of a hypothetical situation where risk levels are chosen with the true calibrated probabilities $c$, underestimated probabilities $\hat{c}_{-} = c - 0.01$, and overestimated $\hat{c}_{+} = c + 0.01$.}
  \vspace{18pt}
\end{figure}

Our main purpose is to develop cautious calibration methods and measure how consistent they are in maintaining underconfidence. However, this theoretical setting is a way to measure the impact of cautiousness on the expected outcomes resulting from a specific decision-making process. Optimal risk level selection is just one example of a scenario where cautiousness is necessary. Further exploration to find other high-risk decision-making frameworks could enhance our understanding of the role and benefits of cautiousness.

\section{Related work}
\label{sec:relatedwork}

To our best knowledge, no works have presented the need to directly use cautiously calibrated probabilities for decision-making. Still, there are methods that try to estimate a confidence interval around the predicted probabilities. This is usually done to evaluate the trustworthiness of the prediction, but a lower value of the confidence interval can be used for building a cautious calibration map. We will discuss several methods and describe the theoretical problems that might occur when using these methods for cautious calibration estimation.

\textbf{Simplified Venn-Abers predictors \citep{vovk2012venn}}. Venn-Abers predictors output multiple probability predictions given one score, one of which is well-calibrated. Getting a Venn-Abers estimate for a test instance requires retraining and re-calibrating the model from the start twice, adding the test instance with both possible labels ($1$ and $0$). This is computationally inefficient and infeasible in the case of LLMs. There is a more efficient (but still inefficient) way to skip the retraining step, but this will lose the validity guarantee about one of the predictions being well-calibrated. This is done with the simplified Venn-Abers (SVA) approach, where the two possible predicted probabilities are obtained by adding an instance with label $1$, re-calibrating with isotonic calibration, and then doing the same with label $0$. We use the label $0$ result as the cautious calibration estimate. The main drawbacks of this method are that it is not conservative enough, as adding one element with label $0$ will not change the calibration map sufficiently, and it does not guarantee validity. Still, the lower estimations can be used as potential cautious calibration estimates.

\textbf{Reliably calibrated isotonic calibration (RCIR) \citep{nyberg2021reliably}}. This method calculates credible intervals for the bins found by isotonic calibration. In the article, it is done to create a more stable isotonic calibration by merging the bins where the size of the credible interval exceeds some threshold, most likely due to the bin containing too few elements. Still, the lower bounds of the calculated intervals, both for original or merged isotonic bins, can be used for cautious calibration estimates. The credible intervals used in this work are calculated using Bayesian statistics approach, where Beta distribution with a uniform prior is used to calculate the posterior distribution for the probability in one isotonic calibration bin. This approach has two problems. Firstly, the use of Bayesian credible intervals depends on the chosen prior. It is typical to use a uniform prior, but as the true prior is unknown, we are unable to provide probabilistic guarantees for these types of intervals. The second problem arises due to having the same lower bound for all elements in the bin, a problem which will be explained in Section \ref{sec:motivation}.

\textbf{Histogram binning with Clopper-Pearson intervals \citep{park2020pac}.} This method is similar to the previous and is meant for finding confidence intervals around the predicted probabilities. The intervals are also calculated in bins, but this time using histogram binning \citep{zadrozny2001obtaining}. The second change is the use of frequentist confidence intervals, more precisely, Clopper-Pearson intervals \citep{clopper1934use}. The histogram binning method shares the same flaws as the isotonic calibration binning approach, while also dependant on the binning parameters. Clopper-Pearson intervals, however, are superior to the Bayesian approach, since they are not dependant on any priors, are known to be conservative and come with certain probabilistic guarantees, which will be discussed shortly.

We also include some classical calibration methods like isotonic calibration \citep{zadrozny2002transforming}, logistic calibration \citep{platt1999probabilistic} and beta calibration \citep{kull2017beta} in our experiments. From the cautious approaches we include SVA and a combination of the last two approaches. We take the good parts of both: isotonic calibration (or RCIR) bins from the first and the Clopper-Pearson confidence intervals from the second, hoping that this gives the strongest method to compare against.

Next, we will introduce our method for cautious calibration, which can be thought of as a generalization of Clopper-Pearson bounds and show how to approach binning to make sure that we are being truly cautious.

\section{Cautious Calibration}

\subsection{Notation}

Before explaining our solution and the motivation behind it, we introduce the setting, which follows the typical binary classification framework. We consider data points drawn i.i.d from $(X, Y) \sim \mathcal{D}$ from the distribution $\mathcal{D}$ over $\mathcal{X} \times \mathcal{Y}$, where $\mathcal{X}$ is the instance space and $\mathcal{Y} = \{0, 1\}$ the label space. We define a scoring model $s: \mathcal{X} \rightarrow \mathbb{R}$ where a larger score $s(X)$ indicates larger confidence towards class 1. We look at the different partitions of the observed data: training data $D_{train} \sim \mathcal{D}$ for training the scoring model, calibration data $D_{cal} \sim \mathcal{D}$ for calibrating the model post-training (similarly to methods like isotonic and logistic calibration) and test data $D_{test} \sim \mathcal{D}$ for evaluation. 
The calibration data $D_{cal} = \{(x_i, y_i)\}_{i=1}^n$ is ordered so that $z_1 \leq z_2 \leq \dots \leq z_n$ where $z_i = s(x_i)$ is the model output score for the instance $x_i$.

\noindent In conclusion, the calibration data is given as follows:

\begin{itemize}[nosep,before=\leavevmode\vspace*{-1\baselineskip}] 
    \item Datapoint vector $\mathbf{x} = (x_1, x_2, \dots, x_n)$
    \item True label vector $\mathbf{y} = (y_1, y_2, \dots, y_n)$
    \item Model output vector $\mathbf{z} = (z_1, z_2, \dots, z_n)$
    \item True unknown calibrated probability vector $\mathbf{c} = (c_1, c_2, \dots, c_n)$ where $c_i = Pr(Y = 1 | z_i)$.
\end{itemize} 

Similarly to classical post-hoc calibration methods, we assume monotonicity between the model output scores $\mathbf{z}$ and the true calibrated values, meaning that $c_1 \leq c_2 \leq \dots \leq c_n$ as well. This is to say that we trust the model's ordering of scores but not their values. The aim of both calibration and cautious calibration methods is to learn the estimates $\mathbf{\underline{\hat{c}}} = (\underline{\hat{c}}_1, \underline{\hat{c}}_2, \dots, \underline{\hat{c}}_n)$ for the true calibrated probabilities based on the true label vector $\mathbf{y}$. However, while traditional methods aim for accurate average estimates, cautious calibration opts for consistently lower estimates to avoid overconfidence (while still avoiding the trivial, but useless case of $\mathbf{\underline{\hat{c}}}=(0,\dots,0)$).

\subsection{Motivation}
\label{sec:motivation}

As mentioned in Section \ref{sec:relatedwork}, Clopper-Pearson intervals are intuitively fitting for cautiousness since they provide conservative confidence intervals with probabilistic guarantees \citep{brown2001interval}. However, these intervals are traditionally applied to a binary vector generated from a Bernoulli process with a constant probability parameter, which in the remark of Definition \ref{def:hetbernoulli} we term as a \emph{homogeneous Bernoulli vector}. This differs from the calibration scenario, where each label in a calibration set is derived from a Bernoulli trial with a unique parameter \(c_k\), defining this as a \emph{heterogeneous Bernoulli vector} in Definition \ref{def:hetbernoulli}.
 Calibration data labels $\mathbf{y}$ or a subsequence of them is a heterogeneous Bernoulli vector, where the underlying probabilities are monotonically increasing.

\begin{definition}
Let $\mathbf{p} = (p_1, p_2, \dots, p_m)$ with each $p_i \in [0,1]$ for $i = 1,\dots, m$. If given a random vector $S^\mathbf{p} = (S^\mathbf{p}_1, S^\mathbf{p}_2, \dots, S^\mathbf{p}_m)$, where $S^\mathbf{p}_i \sim Bernoulli(p_i)$, then we call $S^\mathbf{\mathbf{p}}$ a \textbf{heterogeneous Bernoulli vector} and denote $S^\mathbf{\mathbf{p}} \sim HBernoulli(\mathbf{p})$.
\label{def:hetbernoulli}
\end{definition}

\begin{remark}
For ease of reference, if $p_1 = p_2 =\dots = p_m$, we call the heterogeneous Bernoulli vector a \textbf{homogeneous Bernoulli vector}.
\label{remark_homog}
\end{remark}

\begin{figure}[t]
\centering
\includegraphics[width=8.5cm]{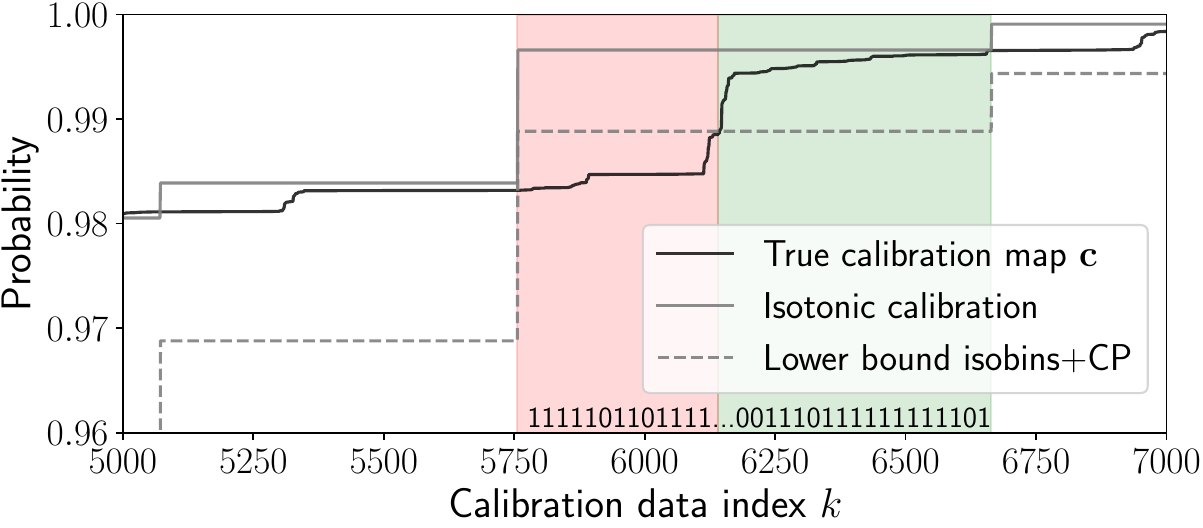}
\caption{A true calibration map (black), isotonic calibration estimate (grey) and a lower bound estimate based on isotonic calibration (dashed grey). The green area shows where the lower bound method holds. The red area shows where the lower bounds are incorrect.}
\label{fig:isobins-cp-problem}
\vspace{18pt}
\end{figure}

A problem arises when treating a subsequence of calibration data labels \(\mathbf{y}_{k-m+1,\dots,k}\) as a homogeneous Bernoulli vector and applying the same lower bound to every element in that subsequence. This is illustrated in Figure \ref{fig:isobins-cp-problem}. The highlighted region represents one isotonic calibration bin. We can see that the true calibrated values (black) are increasing inside the bin, but both isotonic calibration (grey) and the corresponding lower bound calculation (dashed grey) will assign equal values to all elements in the bin. This bound will hold for the rightmost element, as will be proved later, and very often also holds for other right side elements of the bin (green region), but will often not hold for the left side elements of the bin (red region), where the lower bound exceeds the true calibrated probability $\underline{\hat{c}}_k > c_k$. This demonstrates the weakness of assigning the same lower bound to an entire bin and motivates our approach: a bin can be used only to calculate the lower bound of it's rightmost element.

Our work advances cautious calibration in two ways. First, we show the correct way of selecting subsequences for lower bound estimation. This ensures that when using methods that rely on inverted hypothesis testing \citep{berger2001chapter9}, like Clopper-Pearson, the results will be conservative and with certain probabilistic guarantees. As the second contribution, we prove that in addition to the Clopper-Pearson interval, there is a wider set of statistic functions to be used with the inverted hypothesis testing approach, which also produce lower bound estimates with similar properties. Clopper-Pearson intervals are the most simple and computationally efficient example of those.

\subsection{Lower Bound Calculation}

\textbf{Selecting Label Subsequences for Cautious Calibration.} As discussed in the previous section, assigning equal lower bound estimates to all elements in a bin can result in estimates that are overconfident. We prove that if we only use a subsequence of labels preceding an element $k$ to calculate the lower bound $\underline{\hat{c}}_k$, then we can probabilistically guarantee that the bound is correct. This will be shown in the end of this subsection in Theorem \ref{theorem:guarantees}.

\begin{figure}[t]
  \centering
  \includegraphics[width=8.5cm]{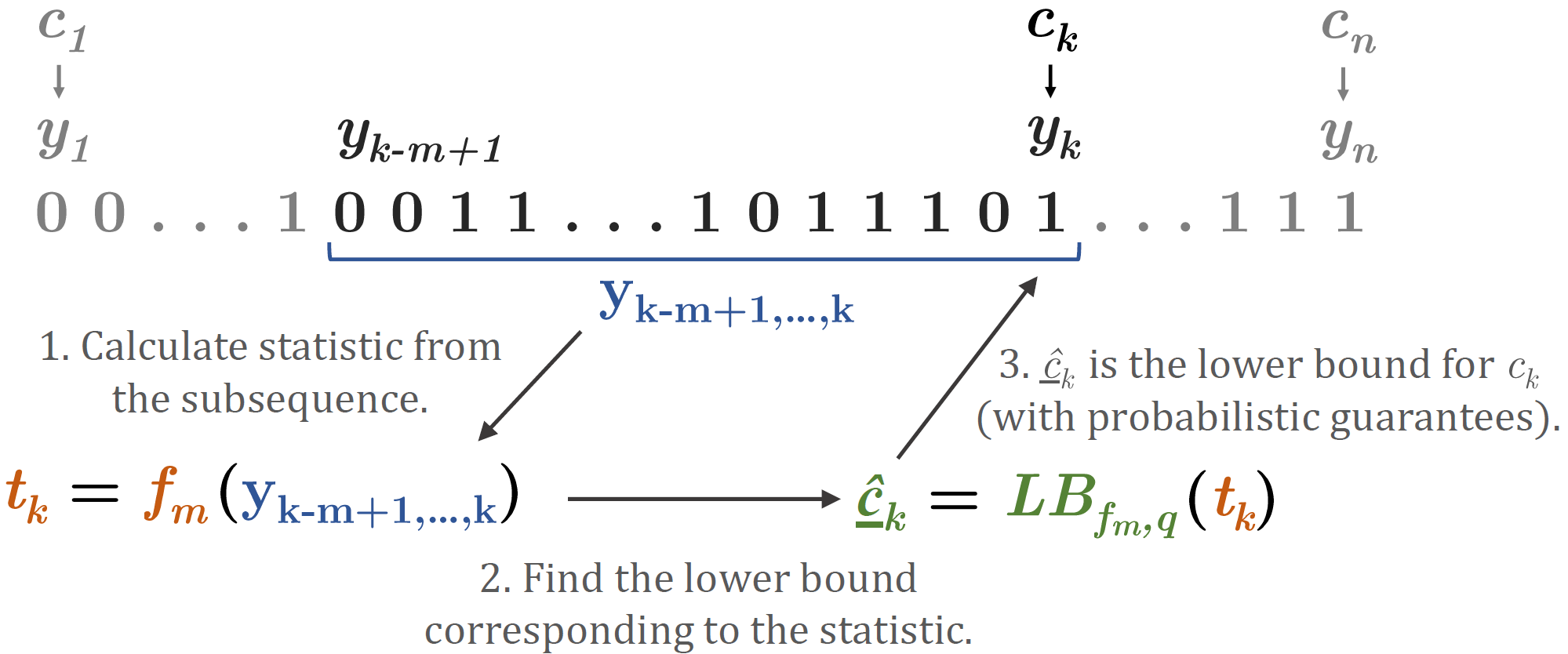}
  \caption{Lower bound $\underline{\hat{c}}_k$ calculation for the true calibrated probability $c_k$. Step 1 calculates a statistic from the binary subsequence $\mathbf{y_{k-m+1,\dots,k}}$. Step 2 finds the lower bound $\underline{\hat{c}}_k$ based on statistic $t_k$. Step 3 assigns $\underline{\hat{c}}_k$ as the lower bound for position $k$.}
  \vspace{18pt}
  \label{fig:calibration_data_lb}
\end{figure}

Figure \ref{fig:calibration_data_lb} demonstrates the process of calculating a lower bound $\underline{\hat{c}}_k$ for a true probability $c_k$. First we select a subsequence of labels \( \mathbf{y}_{k-m+1,\dots,k} \) of length \(m\) preceding $k$. We then calculate a statistic $t_k$ with a function $f_m$ that assigns a real value to a binary vector of length $m$ (for the sake of clarity, we will refer to any function calculating a statistic as a \textbf{statistic function} throughout this article). In the case of Clopper-Pearson intervals, this function $f_m$ is the sum of ones in the selected subsequence. The statistic is then mapped to the corresponding lower bound estimate $\underline{\hat{c}}_k$, calculated with the lower bound function $LB_{f_m, q}$, which will be defined later. As every $k$ has its own preceding subsequence, lower bounds are calculated separately for each \(k\) from its respective subsequence. Meaning that $\underline{\hat{c}}_k$ is a function of a subsequence of labels $y_{k-m+1,\dots,k}$, where $1\leq k-m +1 < k \leq n$ and $k,m\in\mathbb{N}$. The algorithm for calculating a lower map for one calibration set is described in Supplementary \ref{algo}.

\textbf{Inverted Hypothesis Testing for Lower Bound Estimation.}  Before explaining which type of guarantees this approach holds, we need to understand the process of lower bound calculation. We will start by explaining the idea behind one-sided Clopper-Pearson intervals \citep{cai2005one} as an example of inverted hypothesis testing for interval calculation. Then, we generalize this approach to work with a wider set of statistic functions.

Suppose we have a homogeneous Bernoulli vector with parameter $p$. In the Clopper-Pearson case, we calculate the sum function statistic on this random vector, which leads to a binomial distribution. This means we can form and test a hypothesis about $p$ using the binomial test. Unlike classical hypothesis testing, which fixes one certain value of $p'$ for the hypothesis (e.g. $H_0$: $p < p'$) and tests it once, inverted testing evaluates every possible $p' \in [0, 1]$ the same way. All values of $p'$ where the null hypothesis is rejected form an interval $[0, \underline{p}']$ of values, for which the alternative hypothesis ($H_1$: $p \geq p'$) is supported, meaning that the true parameter is likely to be larger than these values. The maximum of these values $\underline{p}'$ is used as a lower bound estimate.

The previous example used homogeneous Bernoulli vectors for calculations, but real-world calibration involves heterogeneous Bernoulli vectors. In theory, it should be possible to construct a new type of confidence interval on these vectors using an inverted hypothesis testing approach, but it would be impractical due to the excessive number of potential hypotheses. Nevertheless, as will be confirmed in Theorem \ref{theorem:guarantees}, constructing our lower bound calculation method on homogeneous Bernoulli vectors and applying it correctly on heterogeneous Bernoulli vectors (taking only left subsequences) provides conservative lower bounds with probabilistic guarantees.

Still, as the homogeneous Bernoulli vector intervals are not tailored for heterogeneous ones, using the standard Clopper-Pearson version might not be the only or the best way for lower bound estimation. We cannot change the homogeneous/heterogeneous approach, but we can exchange the sum statistic function for another type of statistic. For example, we could consider a new recursive statistic function that examines all smaller subsequences than a given size \(m\), calculates their lower bounds, and selects the maximum value as the final statistic. This approach can avoid some of the problems caused by fixing just a single $m$, as that might not be the subsequence giving the best or largest lower bound.

Motivated by this, we generalize the inverted hypothesis testing method to work with all statistic functions that are \textbf{monotonic w.r.t 0$\,\shortrightarrow$1 bit-flipping} (definition in Supplementary \ref{proofs}). These are a group of functions $f_m: \{0, 1\}^m \rightarrow \mathbb{R}$ that maintain or increase its value when any 0 in its input is flipped to a 1. An example would be the sum function, which returns 2 in the case of vector (0, 1, 1, 0) and 3 in the case of vector (0, 1, 1, 1), where the last 0 is flipped to 1. All such functions can be used instead of the sum function to get different cautious calibration estimations with the same probabilistic guarantees.

\textbf{Lower Bound Calculation and Probabilistic Guarantees.} The lower bound calculation with the inverted hypothesis testing approach is formalized in Definition \ref{def:lowerbound}. Intuitively, the set in the definition consists of the aforementioned set of probabilities $p$ for which we reject the null hypothesis, and, as stated, we choose the maximum of those for our lower bound.
The definition includes the \emph{CDF for the statistic function $f_m$ of a heterogeneous Bernoulli vector $S^{\mathbf{p}}$}, specified in Definition \ref{def:hetbernoullicdf}. This is essentially the way to describe our null distributions, as it represents the cumulative distribution function for our chosen statistic, given that the data is generated with the underlying probability $p$.

\begin{definition}
We define a \textbf{CDF for the statistic function $f_m$ of the heterogeneous Bernoulli vector  $S^\mathbf{p} \sim HBernoulli(\mathbf{p})$} as $$F_{f_m,\mathbf{p}}(t) := Pr(f_m(S^\mathbf{p}) < t),$$ where $t \in \mathbb{R}$ and $\mathbf{p} \in [0,1]^m$.
\label{def:hetbernoullicdf}
\end{definition}

\begin{definition}
We define a \textbf{lower bound function} for any $t\in\mathbb{R}$, a fixed probability level $q$ and statistic function $f_m$ as 
\begin{align*}
    LB_{f_m,q}(t) :&= \max\,\{\,p\, |\,F_{f_m,\mathbf{p}}(t) \geq q ,\,  \mathbf{p} = (p, \dots, p), p\in[0,1] \}
\end{align*}

\label{def:lowerbound}
\end{definition}

The bounds derived through inverted hypothesis testing are valuable for cautious calibration because they are conservative and provide probabilistic guarantees linked to hypothesis testing. These guarantees can be expressed as follows: Given a confidence level \( q \), if the true calibrated value \( c \) were lower than our estimated lower bound \( \underline{\hat{c}} \), then the probability of observing a statistic as high or higher from our data distribution \( D^{\mathbf{c}} \) would be at most \( 1 - q \). This guarantee applies to any statistic function that is monotonic with respect to 0$\,\shortrightarrow$1 bit-flipping. This concept is rigorously outlined in Theorem \ref{theorem:guarantees}, with a detailed proof available in Supplementary \ref{proofs}.

\begin{theorem}
\label{theorem:guarantees}
Let $\mathbf{c}=(c_1,\dots,c_m)$ be a monotonic probability vector and let $D^{\mathbf{c}} \sim HBernoulli(\mathbf{c})$ be a heterogeneous Bernoulli vector of length $m$. If statistic function $f_m$ is monotonic w.r.t. 0$\,\shortrightarrow$1 bit-flipping then for any fixed probability level $q\in[0,1]$ and any statistic value $t\in\mathbb{R}$ it holds that:
$$\Pr\,[\,f_m(D^{\mathbf{c}}) \geq t \,|\, c_m < LB_{f_m,q}(t)\,] \leq 1 - q.$$ 
\end{theorem}

To illustrate with a specific example, let's consider using the sum function \( f_m \) with a window length \( m = 1000 \) and a confidence level of \( q = 0.99 \). Suppose in our data we observe a sequence consisting of $999$ ones and $1$ zero, yielding a statistic \( t = 999 \). We can then estimate a lower bound \( \underline{\hat{c}} \approx0.993 \). With this, we would assert that if the actual calibrated probability \( c \) would be below $0.993$, then the probability of seeing a statistic as high or higher than $999$ would be less than $0.01$. Therefore, it is highly probable that the true value of \( c \) is at least \( \underline{\hat{c}} \), confirming that we have a reliable lower bound estimation for this scenario.

\subsection{Practical Calculation of Lower Bounds}

We will be using two different statistic functions for lower bound calculation. Given an arbitrary binary vector $y\in \{0,1\}^m$ and $m\in\mathbb{N}$, the first is the sum function $f_m^{\text{sum}}(y) := \Sigma_{i=1}^m y_i$, making the applied method equivalent to the Clopper-Pearson method. The other is called max-cp, namely 
$$
    f_{m_1, m_2}^{\text{max-cp}}(y)\, {=}\max\{LB_{f_j^{\text{sum}},q}(f_j^{\text{sum}}(y_{m_2-j+1,\dots,m_2})),m_1\,{ \leq }\,j\, {\leq}\, m_2\}
$$
where the max-cp statistic function looks at all subsequences from size $m_1$ to $m_2$ ($m_1 \leq m_2$) ending at the last element, calculates Clopper-Pearson lower bounds for all and takes the maximum of those as the estimate. As proven in Supplementary \ref{proofs}, both statistic functions are monotonic w.r.t. 0$\,\shortrightarrow$1 bit-flipping, making them suitable for estimating lower bounds with probabilistic guarantees.

We will describe the practical calculation of both of these bounds. These approaches differ in many aspects, since for the easier sum statistic functions, the used CDFs can be described analytically, making the computations exact and fast. For the max-cp statistic we need to calculate approximate CDFs with simulations, as the distribution is much more complex.

For both methods, we employ the left subsequence approach with inverted hypothesis testing for lower bound calculation. For a clearer understanding of the experiment results, we introduce the following acronyms: HTLB (Hypothesis Testing Lower Bounds) denotes the general method using one-sided inverted hypothesis testing confidence intervals on left subsequences, HTLB+CP refers to the Clopper-Pearson a.k.a. sum statistic function, and HTLB+MAXCP applies to the max-cp statistic.

\textbf{Practical Calculation of HTLB+CP Lower Bounds.} Calculating the lower bound function for the sum statistic is straightforward due to the direct relationship between the Beta and Binomial distributions in case of Clopper-Pearson intervals \citep{brown2001interval}. The lower bound function is defined as:
$$LB_{f_m^{\text{sum}}, q} = q\text{'th quantile of } Beta(\alpha=t, \beta=m - t + 1),$$
where $t$ is our sum statistic, i.e., the number of 1s in our subsequence, $m$ is the subsequence length, making $m - t$ equal to the nr of 0s in the sequence. This quantile can be easily obtained with the percent-point function of the same Beta distribution.

\textbf{Practical Calculation of HTLB+MAXCP Lower Bounds.} The max-cp statistic function, being more complex, lacks an analytical solution for direct calculation. This means that we must create our own empirical distributions for the inverted hypothesis testing for as many $p$ values as we can and precalculate a mapping between statistic values and corresponding lower bounds. The precalculation has to be done for each $m_1$ and $m_2$ pair separately, which is computationally intensive. However, once the precalculation exists, the using of the lower bounds only requires using the existing mapping. The main steps for calculating a lower bound mapping for $m_1$, $m_2$, $q$ and $f_{m_1, m_2}^{\text{max-cp}}$ are as follows:

\begin{enumerate}
    \item We select as many values of $p$ as we are able to use for calculations.
    \item For each $p$, we sample homogeneous Bernoulli vectors, calculate statistics with  $f_{m_1, m_2}^{\text{max-cp}}$ on them and end up with an empirical distribution over the statistic values for each $p$.
    \item For each $p$, we find the statistic value that is the $q$'th quantile of its empirical distribution.
    \item We reverse the mapping so that each statistic value is mapped with the minimum $p$ for which it is the $q$'th quantile. \footnote{We take the minimum value that is approximately equal to $q$, which is slightly different from the lower bound definition but adds cautiousness in case we have misestimations due to approximate calculations.}
\end{enumerate}

\section{Experiment setup}
\label{sec:exp_setup}

\subsection{Data generation and evaluation}

On real data, the theoretical guarantees hold, but due to the true calibrated values being unknown, we cannot evaluate or compare the cautiousness of different methods. This means that we have to generate our own data where the ground truth is known. The data generation process consists of two parts: generating true calibration maps and generating calibration sets.

\textbf{True calibration map generation.} The true calibrated probability $c_k = Pr(Y=1|z_k)$ is the probability of observing class 1 when the predicted score is $z_k$. Although theoretically, the true calibration map is continuous, in practice, we always estimate it on our finite calibration set. This is why we represent the true calibration map as a vector of true probabilities $\mathbf{c}$ corresponding to our calibration data set points. As stated before, we assume that there is a monotonic relationship between $z_k$ and $c_k$. We want to generate maps with a broad coverage over different shapes, for which we use a simple recursive algorithm described in \cite{allikivi2019non}. An example of 100 generated maps can be found in Supplementary \ref{supfigs}. The maps are generated for high probabilities between 0.9-1, as stated before, since decision-making in high-risk scenarios takes place only in case of high certainty.

\textbf{Calibration set generation.} Next, we need a calibration set in order to be able to learn an estimated calibration map and compare it with the true one. If the true calibration map is known, the calibration set can be thought of as a realization of a heterogeneous Bernoulli vector with parameter $\mathbf{c}$. This means that to obtain one calibration set $\mathbf{y}$, for each $k$, we sample from $y_k \sim Bernoulli(c_k)$ to obtain a label for the $k$'th element. We can sample multiple different calibration label vectors from one true calibration map, learn estimated calibration maps on them and see the behaviour of the calibration algorithm w.r.t. the true calibrated values.

\subsection{Methods}

\textbf{Calibration methods.} First, we include some classical calibration methods in our experiment that are meant for estimating true calibrated probabilities. These methods include isotonic calibration (\textbf{isocal}) \citep{zadrozny2002transforming}, logistic calibration (\textbf{logcal}) \citep{platt1999probabilistic} and beta calibration (\textbf{betacal}) \citep{kull2017beta}. Label smoothing \citep{goodfellow2016deep} has been applied to isotonic and logistic calibration due to its usefulness in reducing overconfidence. Although these methods don't try to be cautious, including them in comparison will help to understand the difference between cautious and classical approaches in these scenarios.

\textbf{Existing methods repurposed for cautious calibration.} Next, we include methods that are meant for quantifying uncertainty about the predicted probabilities and can be repurposed for cautious calibration. The first method, the only one in this section used directly in its original form, is lower bound estimation from simplified Venn-Abers predictors (\textbf{SVA}) \citep{vovk2012venn}. As written earlier, this method is very close to isotonic calibration but tends to prefer underconfidence. The second method is a combination of two previously mentioned approaches. We take the isotonic calibration binning from \citep{nyberg2021reliably} and the Clopper-Pearson idea from \citep{park2020pac} and combine them into one method we call \textbf{isobins+CP}. The third method is very similar to the last one, where instead of isotonic calibration bins, we use the reliably calibrated isotonic calibration (RCIR) \citep{nyberg2021reliably} bins, where too small bins have been joined to get a more reliable prediction. This is also paired with Clopper-Pearson lower bounds and is called \textbf{RCIR+CP}.

\textbf{Our methods for cautious calibration.} We use our \textbf{HTLB+CP} and \textbf{HTLB+MAXCP} approaches with subsequence size 2000. This choice was made based on empirical evaluation, as it provided quite stable but not overly cautious maps, demonstrated in Supplementary \ref{supfigs}. Small subsequence sizes were excluded since it is not possible to predict high enough lower bounds with too little evidence. Too large sizes can also be overly cautious on our 10K sized datasets. For \textbf{HTLB+MAXCP}, we used subsequence sizes between 100 ($m_1$) and 2000 ($m_2$) in order to exclude the very small subsequences.

\textbf{Runtime and complexity.} Most of the methods we have used for comparison are either linear or loglinear (isocal, logcal, betacal, isobins+CP and RCIR+CP). Our methods run in $O(mn)$, where $n$ is the sequence length and $m$ the subsequence length. In practice, calculating the map with our methods on 10000 instances takes approximately 1-3 seconds, which is a small overhead compared to $<$ 1 second runtime for the linear/loglinear methods. SVA has quadratic complexity and is clearly slower compared to other methods. While running HTLB+MAXCP is fast, it requires precalculation of the lower bound map, which is computationally expensive. Complexity of this precalculation is $O(N_p \cdot N_{seq} \cdot m_2 \cdot k)$ where $N_p$ is the number of different probabilities we want to have the precalculated map for, $N_{seq}$ is the number of sequences we want to sample for our empirical distribution, $m_2$ is the maximum subsequence length and $k$ is the statistic function calculation complexity. From our experiments we observed that taking at least a few hundred or a few thousand for $N_p$ and 10K-100K for $N_{seq}$ gave stable results when the data size was 10000 and $m_2$ was 2000. This can of course vary depending on the data and task. Even though it is computationally expensive, precalculation needed for HTLB+MAXCP method has to be only calculated once and can be parallelized easily, making it realistic to use in practice. The code for all mentioned calibration and cautious calibration methods can be found in Supplementary \ref{code}.

\textbf{Post-processing of learned maps.} There are two ways we post-process the learned maps for a more thorough comparison. Since our method has a theoretical maximum limit for the highest lower bound it can predict (a sequence with 2000 1s will produce a lower bound $\sim 0.9978$ with both of our methods), we will have a \textbf{cut} version of all maps, where the maximum value has been clipped to equal our method's maximum value. This will make the comparison of the methods less dependent on the chosen window size. Another modification applied to all cautious calibration methods is forcing them to be monotonic (\textbf{mono}). Due to how cautious calibration methods work, the monotonicity might not be preserved. We apply a simple and conservative algorithm that makes the maps monotonic by going through the map from right to left (from $k=n$ to $k=1$), and once a value $\underline{\hat{c}}_k$ has been observed, all previous values that exceed it $\underline{\hat{c}}_{k'} > \underline{\hat{c}}_k$ where $k' < k$, will be clipped to be at most $\underline{\hat{c}}_k$. We also look at the combination of cutting and forced monotonicity (\textbf{cut+mono}) for these methods. All aforementioned post-processing methods can only lower the values of the initial map, making the results more conservative. Some methods do not need to be cut as they are already cut by construction and the same applies for monotonicity. Some methods can undergo both cutting and monotonic adjustments, making the result even more cautious. Table \ref{tab:methods} summarizes all methods we have used in our experiments together with the applied post-processing options. Note that the rightmost tick for every method represents the the version that leads to the most conservative result, which will be used in the experiment results when referred to the conservative version of a method.

\begin{table}[b]
\caption{Methods used in the experiments. Methods that are monotonic by construction are marked with -$^{\scriptscriptstyle 1}$, and methods that are bounded by construction are marked with -$^{\scriptscriptstyle 2}$. Label smoothing is represented with \emph{l.s.}}
\vspace{15pt}
\centering
\setlength{\extrarowheight}{1pt}
\setlength{\tabcolsep}{4pt}
\begin{tabular}[htbp]{|c|c|c|c|c|c|}
\hline
\multirow{2}{*}{} & \multirow{2}{*}{Method} & \multicolumn{4}{c|}{Post-processing}\\
\cline{3-6}
& & None & Cut & Mono & \makecell{Cut+Mono} \\
\hline
\hline
\multirow{3}{*}{\makecell{Existing\\methods for\\calibration}} & \makecell{\textbf{isocal }\textit{(l.s.)}} & \checkmark & \checkmark &-$^{\scriptscriptstyle 1}$& -$^{\scriptscriptstyle 1}$\\
\cline{2-6}
& \makecell{\textbf{logcal }\textit{(l.s.)}} & \checkmark & \checkmark &-$^{\scriptscriptstyle 1}$ & -$^{\scriptscriptstyle 1}$\\
\cline{2-6}
& \makecell{\textbf{betacal}} & \checkmark & \checkmark &-$^{\scriptscriptstyle 1}$ & -$^{\scriptscriptstyle 1}$\\
\hline
\hline
\multirow{3}{*}{\makecell{Repurposed\\for cautious\\calibration}} & 
\makecell{\textbf{SVA}} & \checkmark & \checkmark & \checkmark & \checkmark \\
\cline{2-6}
& \makecell{\textbf{isobins+CP}} & \checkmark & \checkmark & \checkmark & \checkmark \\
\cline{2-6}
& \makecell{\textbf{RCIR+CP}} & \checkmark & \checkmark & \checkmark & \checkmark \\
\hline
\hline
\multirow{2}{*}{\makecell{Our cautious\\calibration}} & 
\makecell{\textbf{HTLB+CP}} & \checkmark & -$^{\scriptscriptstyle 2}$ & \checkmark & -$^{\scriptscriptstyle 2}$ \\
\cline{2-6}
& \makecell{\textbf{HTLB+MAXCP}} & \checkmark & -$^{\scriptscriptstyle 2}$ & \checkmark & -$^{\scriptscriptstyle 2}$ \\
\cline{2-6}
\hline
\end{tabular}
\label{tab:methods}
\end{table}
\subsection{Experiment setup and evaluation}

\textbf{Evaluation approach.} We have two central concepts we want to evaluate. First, we measure if the cautious estimates $\mathbf{\underline{\hat{c}}}$ are truly cautious w.r.t. $\mathbf{c}$. We calculate the violation percentage that shows how often the lower bounds are incorrect, i.e. how often $\underline{\hat{c}}_k > c_k$. Secondly, we measure how useful cautious calibration is in our example scenario. We evaluate how risk levels chosen with the imperfect estimations $\underline{\hat{c}}_k$ influence the expected outcome, knowing that the true calibrated probability was $c_k$.

\textbf{Experiment setup.} We use 100 different true calibration maps in our experiments and generate 500 calibration sets for each true map, resulting in 50K generated calibration sets. We use 22 variations of 8 conceptually different methods (see Table \ref{tab:methods}) to learn calibration maps or cautious calibration maps on each of these 50K cases. This sums up to 1.1 million learned maps. Results with detailed descriptions will be introduced in the next section of this paper.

\section{Experiment results}

\textbf{Measuring cautiousness with guarantees.} Our first measurements will validate if the probabilistic guarantees hold for our methods. For the guarantees to hold, the measurements have to be made on independent estimations, which is why we cannot provide guarantees for lower bounds calculated on overlapping subsequences. Thus, we select a random $k$'th position for each of the 50K maps and, for each method, check if the lower bound estimation is wrong. We used $q=0.99$ in our experiments, so in order for the guarantees to hold, the percentage of lower bound violations has to be less than 1\%.
Table \ref{tab:guaranteeresults} shows that guarantees hold for our methods, even when they don't have conservative post-processing applied to them. Isobins+CP is also performing well after post-processing, but still, it has a disadvantage due to the aforementioned problems with subsequence choosing. All other methods, especially classical calibration methods, are very frequently non-cautious.

\begin{table}[b]
\caption{Incorrect lower bound percentages over 50K independent estimations. Every estimation is just a single (random) position from the calibration map to ensure independence between estimations.}
\vspace{15pt}
\centering
\setlength{\extrarowheight}{1pt}
\setlength{\tabcolsep}{4pt}
\begin{tabular}[htbp]{|c|c|c|}
\hline
\multirow{2}{*}{Method} & \multicolumn{2}{c|}{Post-processing}\\
\cline{2-3}
& Conservative & None \\
\hline
\hline
\textbf{HTLB+CP} & \textbf{0.0067\%} & \textbf{0.1867\%} \\
\hline
\textbf{HTLB+MAXCP} & \textbf{0.0067\%} & \textbf{0.6311\%} \\
\hline
\textbf{isobins+CP} & \textbf{0.6756\%} & 2.0133\% \\
\hline
\textbf{RCIR+CP} & 8.4900\% & 8.9233\% \\
\hline
\textbf{SVA} & 38.6333\% & 40.2889\% \\
\hline
\textbf{betacal} & 46.3333\% & 51.1489\% \\
\hline
\textbf{isocal} & 49.1000\% & 55.3489\% \\
\hline
\textbf{logcal} & 49.3778\% & 50.3511\% \\
\hline
\end{tabular}

\label{tab:guaranteeresults}
\end{table}

\textbf{Measuring cautiousness without guarantees.} Looking at independent estimates provides guarantees, but practically, we are more interested in the violation percentages within one learned map. In that case, the independence assumption doesn't hold, and no guarantees can be given. The intuitive reason is that we cannot guarantee or protect against getting a very unfortunate calibration set (e.g. worst case, all 1s), so we cannot guarantee a violation percentage < 1\% for all learned maps. Even so, we can still measure the violation percentage within all 50K maps and see if our more cautious approach is empirically giving better results. These results are shown in Figure \ref{fig:violations_across_maps}, where the results for the most conservative version of each method are shown \footnote{This is done since our method would have an even bigger advantage if we did not apply post-processing}. Our methods are clearly the most cautious, with more than 99\% of learned maps having 0 violations, followed by  isocal+CP method. There were some outliers, even with our methods, where up to 17\% of violations occurred within a single map. But compared to other methods, this is still considerably more cautious.

\begin{figure}[t]
  \centering
  \hspace*{-0.5cm}
  \includegraphics[width=8cm]{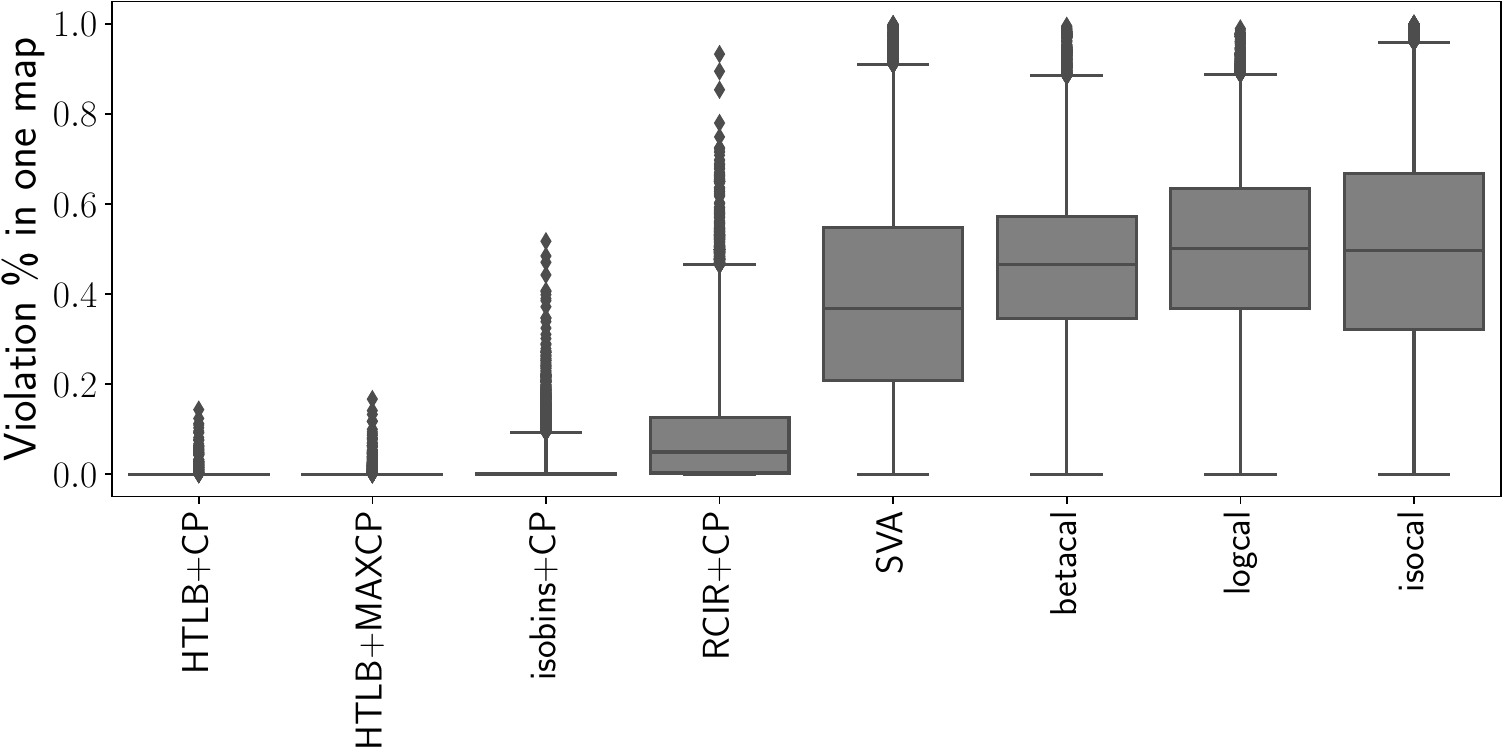}
  \caption{Boxplots for the violation percentages in each of the 50K learned maps for all methods (conservative versions).}
  \vspace{18pt}
  \label{fig:violations_across_maps}
\end{figure}

\textbf{Measuring expected outcome.} The reason for having cautious estimates is to avoid big negative outcomes for all score groups in high-risk decision-making tasks. We measure the effect of cautious calibration on optimal risk level selection scenario, where the risk level is selected based on our estimates $\underline{\hat{c}}_k$ and the expected outcome is calculated, knowing that the true calibrated probability is actually $c_k$. As our goal is to have good expected outcomes for all score groups, then, in principle, we want to find the worst expected outcome in every map and make sure that it's not an extremely negative value. Since the worst value is very dependent on how much data we have used in our experiments, we opt for the 1st percentile worst outcome in our comparisons. Then we can make claims such as 1\% of score groups will have worse results than our reported 1st percentile expected outcome. Figure \ref{fig:outcome_results} shows the results for this, where higher values in the plot stand for better "bad cases". Our methods only have non-negative results on this plot, meaning that even if there are some negative outcomes, they appear rarely. MAXCP approach is outperforming the CP approach, hinting that while still cautious, MAXCP might give estimates closer to the true calibration line. An interesting observation can be made about logcal, betacal, SVA and isobins+CP. While the latter was outperforming logcal, betacal and SVA in cautiousness, it seems like it doesn't only matter how often you make mistakes but also where the mistakes are. It might just be that the problematic subsequence selection in isobins+CP leads to just a few incorrect lower bounds, but once they happen, the amount of overestimation is very large and can cause large negative expected costs. Other methods benefit largely from the cut post-processing, as without cutting, they make many overconfident mistakes in the high-probability region. In Supplementary \ref{supfigs}, similar results for the mean outcome are provided, showing that, as expected, cautiousness reduces the mean outcome. This is the trade-off for providing acceptable results for each score group vs having the best outcome on average.

\begin{figure}[t]
  \centering
  \hspace*{-0.1cm}
  \includegraphics[width=8cm]{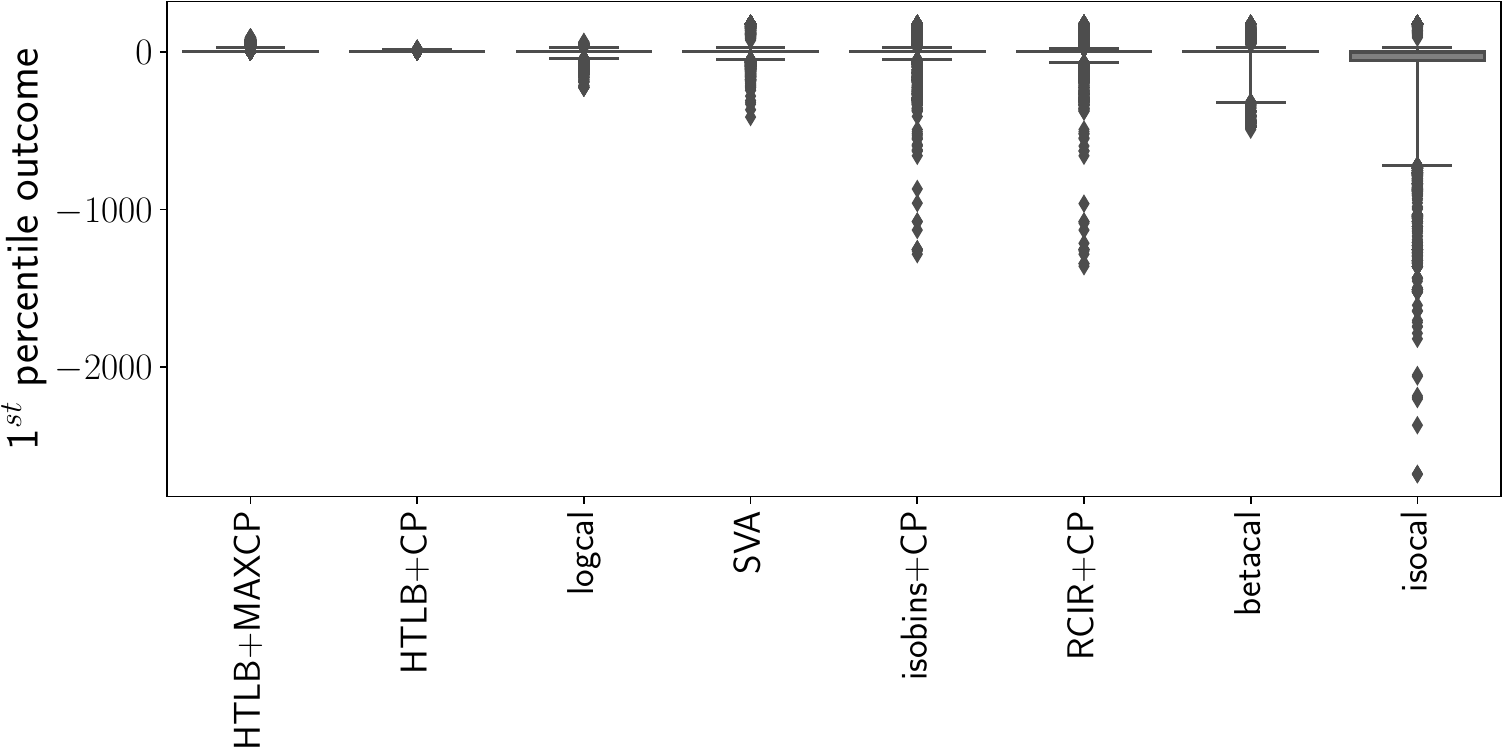}
  \caption{Boxplots for the 1st percentiles of expected outcomes in each of the 50K learned maps for all methods (conservative versions).}
  \vspace{18pt}
  \label{fig:outcome_results}
\end{figure}

\section{Discussion and Conclusion}

This work introduces a novel concept of cautious calibration, which is useful when the aim is to avoid extremely negative outcomes for all score groups consistently.
We have proposed our own theoretically justified approach for cautious calibration and compared it against other calibration methods, as well as existing methods, that are combined or repurposed for cautious calibration.
We demonstrate that our methods are the most cautious ones in terms of correct lower bounds and are useful in avoiding extremely low expected costs in the example scenario.
There is a potential for further improvements in the future, by moving from user-provided subsequence sizes to adaptable sizes, and finding more choices for the statistic function that would lead to analytic distributions, supporting fast computations like in the HTLB+CP approach.
In summary, our research emphasizes the importance of looking further from aggregated performance measures towards models that are more trustworthy and, thus, better usable. Cautious calibration is one contribution towards this goal. Moving forward, we aim to identify scenarios where cautiousness is beneficial and develop methods most suitable for these contexts.

\begin{ack}
This work was supported by the Estonian Research Council un- der grant PRG1604, by the Estonian Centre of Excellence in Artificial Intelligence (EXAI), funded by the Estonian Ministry of Education and Research, and co-funded by the European Union and Estonian Research Council via project TEM-TA119.
\end{ack}



\newpage

\bibliography{mybibfile}

\appendix
\onecolumn

\begingroup

\appendix

\begin{center}
    \vspace{20pt} 
    \textbf{\Huge{Supplementary Material}}
    \vspace{20pt} 
\end{center}

\section{ALGORITHM}\label{algo}

The algorithm for lower bound map calculation with the HTLB method.

\begin{algorithm}[H]
\caption{Calculating HTLB lower bound map.}
\label{alg:applying_on_cal_data}
\begin{algorithmic}[1]
\REQUIRE window length $m$
\REQUIRE statistic function $f_m$
\REQUIRE confidence level $q$
\REQUIRE calibration data labels $\mathbf{y}$ \COMMENT{sorted by scores $\mathbf{z}$}
\STATE Initialize $lower\_bounds$ \COMMENT{empty list of length $n$}
\FOR{$k = m$ to $n$}
    \STATE $\mathbf{y_{k-m+1,\dots,k}} = (y_{k-m+1}, y_{k-m+2}, \dots, y_{k-1}, y_k)$ \COMMENT{select $m$ true labels up to $k$}
    \STATE $t_k = f_m(\mathbf{y_{k-m+1,\dots,k}})$ \COMMENT{statistic value for the selected region}
    \STATE $\underline{\hat{c}}_k = LB_{f_m,q}(t_k)$ \COMMENT{calculate the lower bound}
    \STATE $lower\_bounds[k] = \underline{\hat{c}}_k$
\ENDFOR
\RETURN $lower\_bounds$
\end{algorithmic}
\end{algorithm}

\section{THEOREMS \& PROOFS}\label{proofs}

\textbf{Optimal Risk Level Selection.}
Let us have a probability prediction function $\hat{c}: \mathcal{X} \rightarrow [0,1]$ that, based on the input, estimates the calibrated probability. We have an outcome function $o: (\mathcal{Y}, [0,\infty)) \rightarrow \mathbb{R}$, that calculates the outcome based on a class label $y$ and a risk level $\xi \in [0, \infty)$. We have selected the following outcome function:

\[
o(y, \xi ) = 
\begin{cases}
    \xi  & \text{if } y = 1, \\
    -\xi ^{l} & \text{if } y = 0,
\end{cases}
\]
where $l \in [1, \infty)$ is a fixed constant describing the imbalance of costs. 

We will derive how to choose a risk level $\xi$  based on calibrated probability $c \in [0, 1)$ (we omit the trivial case of $c=1$). We do it by choosing the risk level that gives the highest expected outcome, assuming the data points are drawn from $Y \sim Bernoulli(c)$. This can be thought of as optimizing the outcome for one score group. We can write it down as:

\begin{equation}
r_{opt}(c) = \arg\max_{\xi}\mathbb{E}_{Y \sim Bern(c)} [o(Y, \xi)]\label{argmaxrisk}
\end{equation}

Let's calculate it:
\begin{align*}
    \mathbb{E}_{Y \sim Bern(c)} [o(Y, \xi)] &=\\
    &=\mathbb{E}[o(Y,\xi)|Y=1]\cdot Pr(Y=1) + \mathbb{E}[o(Y,\xi)|Y=0]\cdot Pr(Y=0)= \\
    &= \xi \cdot c + (-\xi^{l})\cdot(1-c)
\end{align*}

Now, let's take the derivative and equalize with zero to find such maximizing $\xi$:
\begin{align*}
(\xi \cdot c + (-\xi^{l})\cdot(1-c))' &= 0 \\
c + (-l\xi^{l-1})\cdot(1-c) &= 0 \\
\xi^{l-1} &= \frac{c}{l(1-c)} \\
\xi &= \Bigg(\frac{c}{l(1-c)}\Bigg)^{\frac{1}{l-1}}. \\
\end{align*}

\textbf{Inverted hypothesis testing approach for lower bound estimation}. Here we have the formal notation for our cautious calibration approach with more details and proofs. For a better following of the proofs, we will also include the definitions here in the supplementary.

We will first define the heterogeneous Bernoulli vector and then accompany it with a remark.

\begin{definition}
Let $\mathbf{p} = (p_1, p_2, \dots, p_m)$ with each $p_i \in [0,1]$ for $i = 1,\dots, m$. If given a random vector $S^\mathbf{p} = (S^\mathbf{p}_1, S^\mathbf{p}_2, \dots, S^\mathbf{p}_m)$, where $S^\mathbf{p}_i \sim Bernoulli(p_i)$, then we call $S^\mathbf{\mathbf{p}}$ a \textbf{heterogeneous Bernoulli vector} and denote $S^\mathbf{\mathbf{p}} \sim HBernoulli(\mathbf{p})$.
\label{def:hetbernoulli}
\end{definition}

\medskip

\begin{remark}
For ease of reference, if $p_1 = p_2 =\dots = p_m$, we call the heterogeneous Bernoulli vector a \textbf{homogeneous Bernoulli vector}.
\label{remark_homog}
\end{remark}

\medskip
Now we define how we categorize the probability vectors.

\begin{definition}
Let $\mathbf{p} = (p_1, \dots, p_m)$ with each $p_i \in[0, 1]$ for $i=1, \dots, m$. When holds that $0 \leq p_1 \leq \dots \leq p_m \leq 1$, the probability vector $\mathbf{p}$ is called \textbf{monotonic}, and in case $p_1 = \dots = p_m$, \textbf{homogeneous}.
Let's note that if we have two monotonic probability vectors $\mathbf{p^{(1)}}$ and $\mathbf{p^{(2)}}$, the inequalities apply pointwise, e.g.  $\mathbf{p^{(1)}} \leq \mathbf{p^{(2)}}$ is defined as $p^{(1)}_i \leq p^{(2)}_i$ $\forall \, i=1, \dots, m$.
    \label{def:monotonicprobvec}
\end{definition}

\medskip
Now, we continue with the definitions of the CDF for the statistic function $f_m$ of the heterogeneous Bernoulli vector and the lower bound.
\begin{definition}
We define a \textbf{CDF for the statistic function $f_m$ of the heterogeneous Bernoulli vector  $S^\mathbf{p} \sim HBernoulli(\mathbf{p})$} as $$F_{f_m,\mathbf{p}}(t) := Pr(f_m(S^\mathbf{p}) < t),$$ where $t \in \mathbb{R}$ and $\mathbf{p} \in [0,1]^m$.
\label{def:hetbernoullicdf}
\end{definition}

\medskip

\begin{definition}
We define a \textbf{lower bound function} for any $t\in\mathbb{R}$, a fixed probability level $q$ and statistic function $f_m$ as 
\begin{align*}
    LB_{f_m,q}(t) :&= \max\,\{\,p\, |\,F_{f_m,\mathbf{p}}(t) \geq q ,\,  \mathbf{p} = (p, \dots, p), p\in[0,1] \}\\
    &=  \max\,\{\,p\, |\,Pr(f_m(S^\mathbf{p}) < t) \geq q ,\,  \mathbf{p} = (p, \dots, p), p\in[0,1] \}.
\end{align*}

\label{def:lowerbound}
\end{definition}

\medskip

Definition \ref{def:bitflipping} introduces a 0$\,\shortrightarrow$1 bit-flipping function. This function gets a binary sequence as input and flips one 0 to 1, e.g. $(0, 1, 0, 0)$ could become $(0, 1, 1, 0)$.

\begin{definition}
\label{def:bitflipping}
For any $k \in \{1, 2, \dots, m\}$ we define the function $g_k: \{0, 1\}^m \rightarrow \{0, 1\}^m$ by 
\[g_k(\mathbf{v})_i=
\begin{cases} 
v_i & \text{, } i \neq k, \\
1 & \text{, } i = k.
\end{cases}
\]
 for any $\mathbf{v}=(v_1,\dots,v_m)\in \{0, 1\}^m$. We refer to the functions $g_1,\dots,g_m$ as the \textbf{0$\,\shortrightarrow$1 bit-flipping} functions.
\end{definition}

\medskip

 Definition \ref{def:monotonicbitflipping} introduces a concept of a function being monotonic with respect to 0$\,\shortrightarrow$1 bit-flipping. This means that when given an original binary sequence $\mathbf{v}$ (e.g. $\mathbf{v} = (0, 1, 0, 0)$) and 0$\,\shortrightarrow$1 bit-flipped sequence $g_k(\mathbf{v})$ (e.g. $g_3(\mathbf{v}) = (0, 1, 1, 0)$), the statistic calculated with function $f$ is larger or equal with the bit-flipped sequence. An example of a function that is monotonic with respect to 0$\,\shortrightarrow$1 bit-flipping is the sum function (e.g. if $f$ is the sum function then $f((0, 1, 0, 0)) = 1$ is smaller than $f((0, 1, 1, 0)) = 2$).

\begin{definition}
\label{def:monotonicbitflipping}
We say that a statistic function $f_m: \{0, 1\}^m \rightarrow \mathbb{R}$ is \textbf{monotonic w.r.t 0$\,\shortrightarrow$1 bit-flipping} if for every $\mathbf{v} \in \{0, 1\}^m$ and every 0$\,\shortrightarrow$1 bit-flipping function $g_k$, where $k\in\{1, 2, \dots, m\}$, it holds that $f_m(\mathbf{v}) \leq f_m(g_k(\mathbf{v}))$.
\end{definition}

\medskip
If we know that our function $f_m$ is monotonic with respect to 0$\,\shortrightarrow$1 bit-flipping, then the heterogeneous Bernoulli CDFs generated with this function have a property which is given in Proposition \ref{def:bitflippingcdfs}. Namely, if we have two monotonic probability vectors $\mathbf{p}$ and $\mathbf{p'}$ where $\mathbf{p} \leq \mathbf{p'}$, the corresponding heterogeneous Bernoulli CDFs have the opposite relationship, where the one defined by higher probability vector is dominated by the one defined by lower probability vector.

\begin{proposition}
\label{def:bitflippingcdfs}
Let the statistic function $f_m$ be monotonic w.r.t. 0$\,\shortrightarrow$1 bit-flipping. For any two monotonic probability vectors $\mathbf{p}$ and $\mathbf{p'}$ where $\mathbf{p} \leq \mathbf{p'}$ it holds that $F_{f_m,\mathbf{p}}(t) \geq F_{f_m,\mathbf{p'}}(t) \; \forall \; t \in \mathbb{R}$.
\end{proposition}

\medskip

\begin{proof}

We are given monotonic probability vectors $\mathbf{p} = (p_1, \dots, p_m)$ and $\mathbf{p'} = (p'_1,  \dots, p'_m)$. Based on $\mathbf{p}$ and $\mathbf{p'}$ let's construct a  vector $\mathbf{\tilde{p}^{(i)}} = (p_1,\dots,p_{i-1},p_{i},p'_{i+1},\dots,p'_m)$, where $i = 1,\dots,m$. By the construction of $\mathbf{\tilde{p}^{(i)}}$, the sequence of inequalities $$\mathbf{p} =\mathbf{\tilde{p}^{(m)}} \leq \dots\leq \mathbf{\tilde{p}^{(i)}}\leq \mathbf{\tilde{p}^{(i-1)}} \leq \dots \leq \mathbf{\tilde{p}^{(0)}} = \mathbf{p'} $$ holds. Namely, for any $i \in \{1,\dots,m\}$ the vectors $\mathbf{\tilde{p}^{(i)}}$ and $\mathbf{\tilde{p}^{(i-1)}}$ differ only at the position $i$ where respectively $p_{i}  \leq p'_{i}$ rendering any such inequality $\mathbf{\tilde{p}^{(i)}}\leq \mathbf{\tilde{p}^{(i-1)}}$ to hold. Furthermore, the constructed vector $\mathbf{\tilde{p}^{(i)}}$ is also monotonic since the right-hand side and the left-hand side until the position $i$ are in themselves monotonic and at the position $i$ the monotonicity is preserved thanks to $p_{i}  \leq p'_{i} \leq p'_{i+1}$. Therefore, we can rewrite our claim as follows:
\begin{equation}\label{sequence}
    F_{f_m, \mathbf{p}}(t) = F_{f_m, \mathbf{\tilde{p}^{(m)}}}(t) \geq\ldots\geq F_{f_m, \mathbf{\tilde{p}^{(i)}}}(t) \geq F_{f_m, \mathbf{\tilde{p}^{(i-1)}}}(t)\geq\ldots \geq F_{f_m, \mathbf{\tilde{p}^{(0)}}}(t) = F_{f_m, \mathbf{p'}}(t).
\end{equation}

\medskip
So, it is sufficient to show that $F_{f_m, \mathbf{\tilde{p}^{(i)}}}(t) \geq F_{f_m, \mathbf{\tilde{p}^{(i-1)}}}(t)$ for any $i\in\{1,\dots,m\}$. Let's fix $i\in\{1,\dots,m\}$ and denote a binary vector that have $0$ in the position $i$ as $\omega = (\omega_1,\dots,\omega_{i-1},0,\omega_{i+1},\dots,\omega_m)$ and a set of such vectors $\Omega_{(i \to 0)}$. Now, we can write based on the Definition \ref{def:hetbernoullicdf}:

\begin{align}
    F_{f_m, \mathbf{\tilde{p}^{(i)}}}(t) =  \Pr(f_m(S^\mathbf{\tilde{p}^{(i)}}) < t)  
     &= \sum_{\omega \in  \Omega_{(i \to 0)}}\Pr(\omega) \mathds{1}_{\{f_m(\omega) < t\}} + \Pr(g_i(\omega)) \mathds{1}_{\{f_m(g_i(\omega)) < t\}} = \label{kolmas1}\\
     &= \sum_{\omega \in \Omega_{(i \to 0)}}\Pr(\omega_{/i})((1 - p_i) \mathds{1}_{\{f_m(\omega) < t\}} + p_i \mathds{1}_{\{f_m(g_i(\omega)) < t\}}),
      \label{kolmas2}
\end{align}
where $\omega_{/i}$ means that $i$'th element is excluded from the vector and
\begin{equation}
    F_{f_m, \mathbf{\tilde{p}^{(i-1)}}}(t) =  \Pr(f_m(S^\mathbf{\tilde{p}^{(i-1)}}) < t) = \sum_{\omega \in \Omega_{(i \to 0)}}\Pr(\omega_{/i})((1 - p_i') \mathds{1}_{\{f_m(\omega) < t\}} + p_i' \mathds{1}_{\{f_m(g_i(\omega)) < t\}}).\label{neljas1}
\end{equation}

\medskip
In the Equation (\ref{kolmas1}), we rewrite the CDF as a sum of occurring probabilities of all possible binary vectors of length $m$ that satisfy the condition $f_m(S^\mathbf{\tilde{p}^{(i)}}) < t$. Specifically, we sum over all vectors $\omega\in\Omega_{(i \to 0)}$ and their 0$\shortrightarrow$1 bit-flipped versions $g_i(\omega) = (\omega_1,\dots,\omega_{i-1},1,\omega_{i+1},\dots,\omega_m)$. For example, the probability $Pr(\omega)$ of drawing one particular vector from $S^\mathbf{\tilde{p}^{(i)}} \sim HBernoulli(\mathbf{\tilde{p}^{(i)}})$ is calculated as the product of the individual Bernoulli probabilities. Lastly, the indicator function preserves only the probabilities if the statistic function $f_m$ of the corresponding vector is smaller than $t$.

\medskip

Let's fix $\omega\in\Omega_{i\to 0}$ and write based on Equations (\ref{kolmas2}) and (\ref{neljas1}):

$$(1 - p_i) \mathds{1}_{\{f_m(\omega) < t\}} + p_i \mathds{1}_{\{f_m(g_i(\omega)) < t\}} \geq (1 - p'_i) \mathds{1}_{\{f_m(\omega) < t\}} + p'_i \mathds{1}_{\{f_m(g_i(\omega)) < t\}}.$$

Now, recalling that vectors $\mathbf{\tilde{p}^{(i)}}$ and $\mathbf{\tilde{p}^{(i-1)}}$ differ only at the position $i$ where respectively $p_{i}  \leq p'_{i}$, in other words, exists $\epsilon \geq 0$ such that $p'_i = p_i + \epsilon$, we can substitute and simplify:

$$(1 - p_i) \mathds{1}_{\{f_m(\omega) < t\}} + p_i \mathds{1}_{\{f_m(g_i(\omega)) < t\}} \geq (1 - (p_i + \epsilon)) \mathds{1}_{\{f_m(\omega) < t\}} + (p_i + \epsilon) \mathds{1}_{\{f_m(g_i(\omega)) < t\}}.$$

$$\epsilon \mathds{1}_{\{f_m(\omega) < t\}} \geq \epsilon \mathds{1}_{\{f_m(g_i(\omega)) < t\}}$$

$$\mathds{1}_{\{f_m(\omega) < t\}} \geq \mathds{1}_{\{f_m(g_i(\omega)) < t\}}$$

Since, $f_m$ is monotonic w.r.t to 0$\shortrightarrow$1 bit-flipping i.e. $f_m(\omega) \leq f_m(g_i(\omega))$ then the inequality holds, meaning that $F_{f_m, \mathbf{\tilde{p}^{(i)}}}(t) \geq F_{f_m, \mathbf{\tilde{p}^{(i-1)}}}(t)$. Where, due to the sequence of inequalities (\ref{sequence}), we get that $F_{f_m, \mathbf{p}}(t) \geq F_{f_m, \mathbf{p'}}(t)$, which is what we wanted to show.
\end{proof}

\medskip

Now we can prove the main result of the paper.

\begin{theorem}
\label{theorem:guarantees}
Let $\mathbf{c}=(c_1,\dots,c_m)$ be a monotonic probability vector and let $D^{\mathbf{c}} \sim HBernoulli(\mathbf{c})$ be a heterogeneous Bernoulli vector of length $m$. If statistic function $f_m$ is monotonic w.r.t. 0$\,\shortrightarrow$1 bit-flipping then for any fixed probability level $q\in[0,1]$ and any statistic value $t\in\mathbb{R}$ it holds that:

$$\Pr\,[\,f_m(D^{\mathbf{c}}) \geq t \,|\, c_m < LB_{f_m,q}(t)\,] \leq 1 - q.$$ 
\end{theorem}

\medskip

\begin{proof}

Let's denote $\underline{\hat{c}} = LB_{f_m, q}(t)$ and define homogeneous probability vector $\mathbf{\underline{\hat{c}}} = (\underline{\hat{c}}, ..., \underline{\hat{c}})$ of length $m$. We notice that by the Definition \ref{def:lowerbound} it also holds that 
\begin{equation}
    Pr(f_m(S^\mathbf{\underline{\hat{c}}}) < t) \geq q.  \label{notice}
\end{equation}
Let's define another probability vector $\mathbf{c_m} = (c_m, ..., c_m)$ also of length $m$.

\hfill

Let's assume that $c_m < \underline{\hat{c}}$.

\hfill

We then know that $c_i \leq c_m \leq \underline{\hat{c}},\, \forall i=1,\dots,m$ and according to Proposition \ref{def:bitflippingcdfs} we get

$$F_{f_m, \mathbf{c}}(t) \geq F_{f_m, \mathbf{c_m}}(t) \geq F_{f_m, \mathbf{\underline{\hat{c}}}}(t)$$

From the Definition \ref{def:hetbernoullicdf} and Equation (\ref{notice}) we get that

$$\Pr(f_m(D^{\mathbf{c}}) < t) \geq \Pr(f_m(S^{\mathbf{c_m}}) < t) \geq \Pr(f_m(S^{\mathbf{\underline{\hat{c}}}}) < t) \geq q$$

$$\Pr(f_m(D^{\mathbf{c}}) < t) \geq q \Rightarrow \Pr(f_m(D^{\mathbf{c}}) \geq t) \leq 1 - q$$

This means that we have shown that 

$$\Pr\,[\,f_m(D^{\mathbf{c}}) \geq t \,|\, c_m < LB_{f_m, q}(t)\,] \leq 1 - q.$$

\end{proof}

Now, let's prove that the sum statistic function is monotonic w.r.t 0$\shortrightarrow$1 bit-flipping.

\begin{proposition}
      A sum statistic function $f_m^{\text{sum}}$ is monotonic w.r.t 0$\shortrightarrow$1 bit-flipping.
     \label{sumstatfun}
\end{proposition}

\begin{proof}
Given an arbitrary binary vector $\mathbf{v}\in \{0,1\}^m$ and its 0$\shortrightarrow$1 bit-flipped version $\mathbf{v}' = g_k(\mathbf{v})$, let's demonstrate that for a function $f_m^{\text{sum}}(\mathbf{s}) := \Sigma_{i=1}^m s_i$ and for any $k\in\{1,\dots,m\}$, it holds that $f_m^{\text{sum}}(\mathbf{v}) \leq f_m^{\text{sum}}(g_k(\mathbf{v}))$:
    $$f_m^{\text{sum}}(\mathbf{v}) = v_k + \Sigma_{i \neq k} v_i \leq 1 + \Sigma_{i \neq k} v_i = f_m^{\text{sum}}(g_k(\mathbf{v}))$$

\end{proof}

\medskip

And lastly, that the max statistic function is monotonic w.r.t 0$\shortrightarrow$1 bit-flipping.
\begin{proposition}
     A max statistic function $f_{m_1, m_2}^{\text{max-cp}}$ is monotonic w.r.t 0$\shortrightarrow$1 bit-flipping.\label{maxstatfun}
\end{proposition}

\begin{proof}
Given an arbitrary binary vector $\mathbf{v}\in \{0,1\}^m$ and its 0$\shortrightarrow$1 bit-flipped version $\mathbf{v}' = g_k(\mathbf{v})$, let's show that for a max statistic function  $$f_{m_1, m_2}^{\text{max-cp}}(\mathbf{v}) := \max\{LB_{f_j^{\text{sum}},q}(f_j^{\text{sum}}(v_{m_1-j+1,\dots,m_2})), j=m_1,\dots,m_2\}$$ it holds that $f_{m_1, m_2}^{\text{max-cp}}(\mathbf{v}) \leq f_{m_1, m_2}^{\text{max-cp}}(g_k(\mathbf{v}))$. 

Let's note that by the Proposition \ref{sumstatfun}, sum statistic function $f_j^{\text{sum}}(\mathbf{s})$ can only increase when changing 0 to 1, i.e. $f_j^{\text{sum}}(\mathbf{v}) \leq f_j^{\text{sum}}(g_k(\mathbf{v}))$ and $\max$ function is also monotonic.

Therefore, it remains to show that $LB_{f_j^{\text{sum}},q}$ is monotonic, i.e. $$p = LB_{f_j^{\text{sum}},q}(t)\leq LB_{f_j^{\text{sum}},q}(t') = p'$$ for any $t\leq t'\in \mathbb{R}$, $q\in[0,1]$ and $j \in \{m_1,\dots,m_2 \}$, which based on Definition \ref{def:lowerbound} means:

$$ p =\max\{p |\,F_{f_j^{\text{sum}},\mathbf{p}}(t)\geq q ,  \mathbf{p} = (p, \dots, p), p\in[0,1] \} \leq \max\{p |\,F_{f_j^{\text{sum}},\mathbf{p}}(t')\geq q , \mathbf{p} = (p, \dots, p), p\in[0,1] \} =p'$$

Let's fix $\mathbf{p} = (p,\dots,p)$. Since, CDFs are non-decreasing functions and $t \leq t'$, we have that $q \leq F_{f_j^{\text{sum}},\mathbf{p}}(t) \leq F_{f_j^{\text{sum}},\mathbf{p}}(t')$. This suggests that $\mathbf{p'}$ is at least $\mathbf{p}$, since by the Proposition \ref{def:bitflippingcdfs} CDF $F_{f_j^{\text{sum}},\mathbf{p}}(t)$ at any given point $t\in\mathbb{R}$ decreases by increasing probability vector $\mathbf{p}$ i.e.
$$p' = \max\,\{\,p\, |\,F_{f_j^{\text{sum}},\mathbf{p'}}(t')\geq q ,\, \mathbf{p} = (p, p, \dots, p), p\in[0,1] \}\geq p$$

Since we fixed $j$ and $q$ arbitrarily, we have that $p' \geq p$ and hence the proposition holds. 

    
\end{proof}

\section{FIGURES}\label{supfigs}

Figure \ref{fig:true_cal_maps} shows an example of 100 calibration maps. The algorithm chooses a random position $k$ and a random value between 0.9 and 1.0 for $c_k$. Then, the same procedure is repeated recursively for elements between $k+1$ and $n$ and for elements between $1$ and $k-1$.

\begin{figure}[H]
  \centering
  \includegraphics[width=10cm]{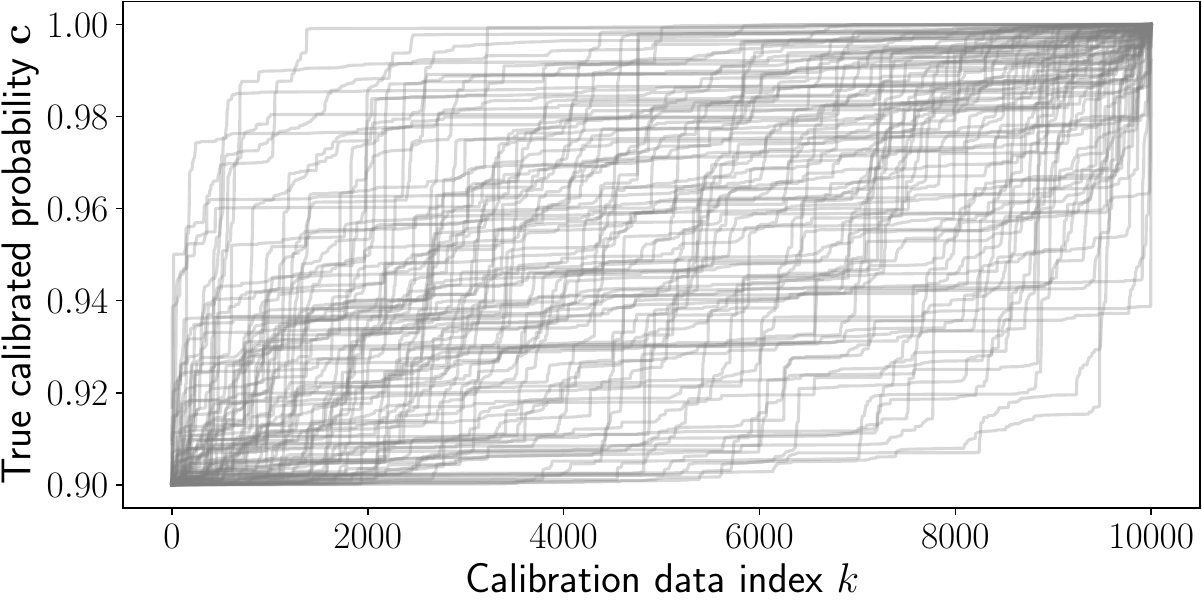}
  \caption{An example of 100 generated true calibration maps.}
  \label{fig:true_cal_maps}
\end{figure}

An example of lower bound estimations with the HTLB+CP method using different window sizes is shown in Figure \ref{fig:subsequence_size}. We can see that small subsequence sizes give estimations that fluctuate a lot, while larger ones are more stable. In our experiments we use size 2000, highlighted in red for this particular example.

\begin{figure}[H]
  \centering
  \includegraphics[width=10cm]{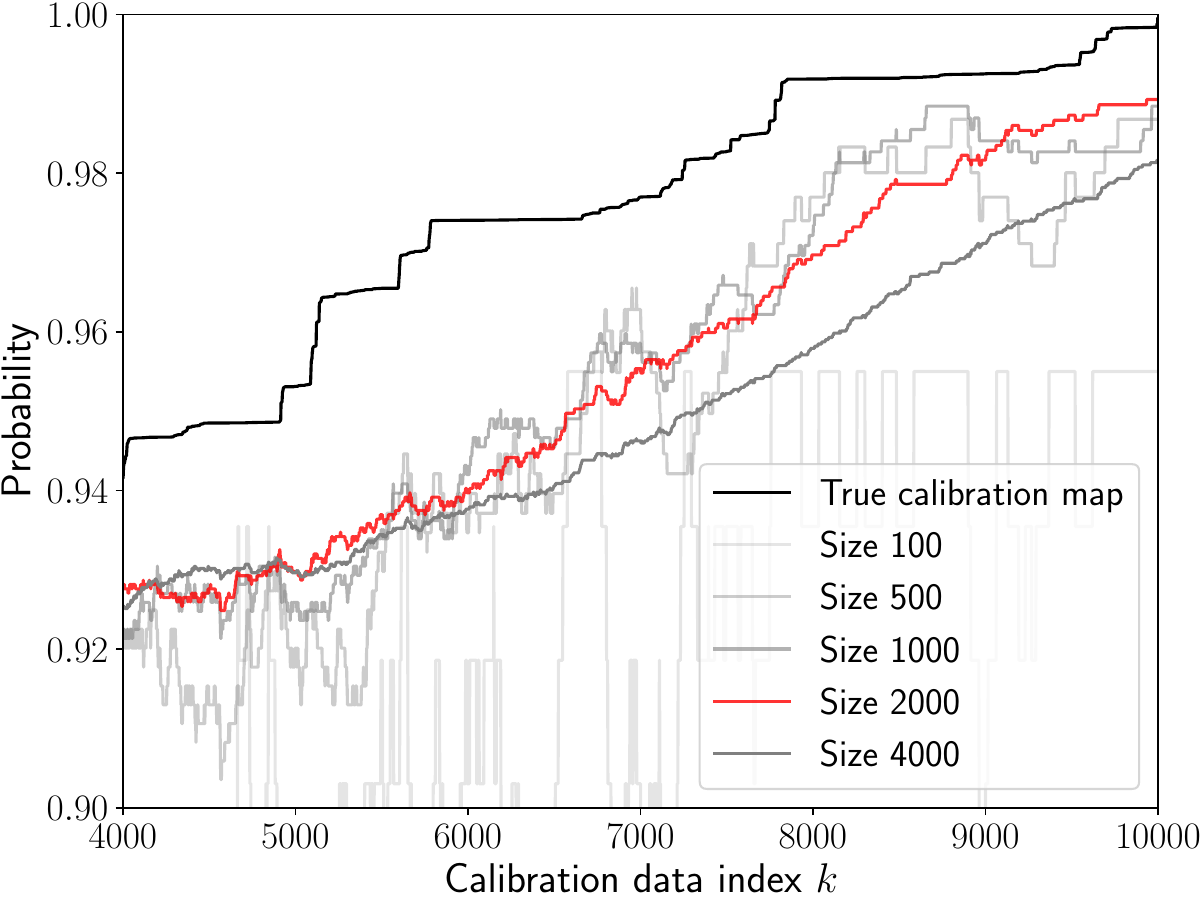}
  \caption{Example of one true calibration map $\mathbf{c}$ (black) and different HTLB+CP estimations with varying subsequence sizes calculated on a calibration set sampled from $\mathbf{c}$}
    \label{fig:subsequence_size}
\end{figure}

The results for the mean outcomes in the example scenario experiments can be seen in Figure \ref{fig:mean_outcome}. We can see that the methods have quite similar results if conservative post-processing has been applied. Without the cutting post-processing the classical calibration methods have both higher and extremely lower outcomes as can be seen in \ref{fig:mean_non_conservative}. This is because the estimations are sometimes very close to the true calibrated probabilities, giving very good outcomes, but can often also be very overconfident, giving extremely bad outcomes. Results for the cautious approaches also change. SVA has both very high and low mean outcomes, HLTB is inferior to RCIR and isobins approaches in the case of mean outcome. This is expected as we are more cautious, and sometimes taking the bigger risk pays off, but as we saw from the 1st percentile results, some score groups will pay for the increase in the average outcome.

\begin{figure}[H]
  \centering
  \includegraphics[width=12cm]{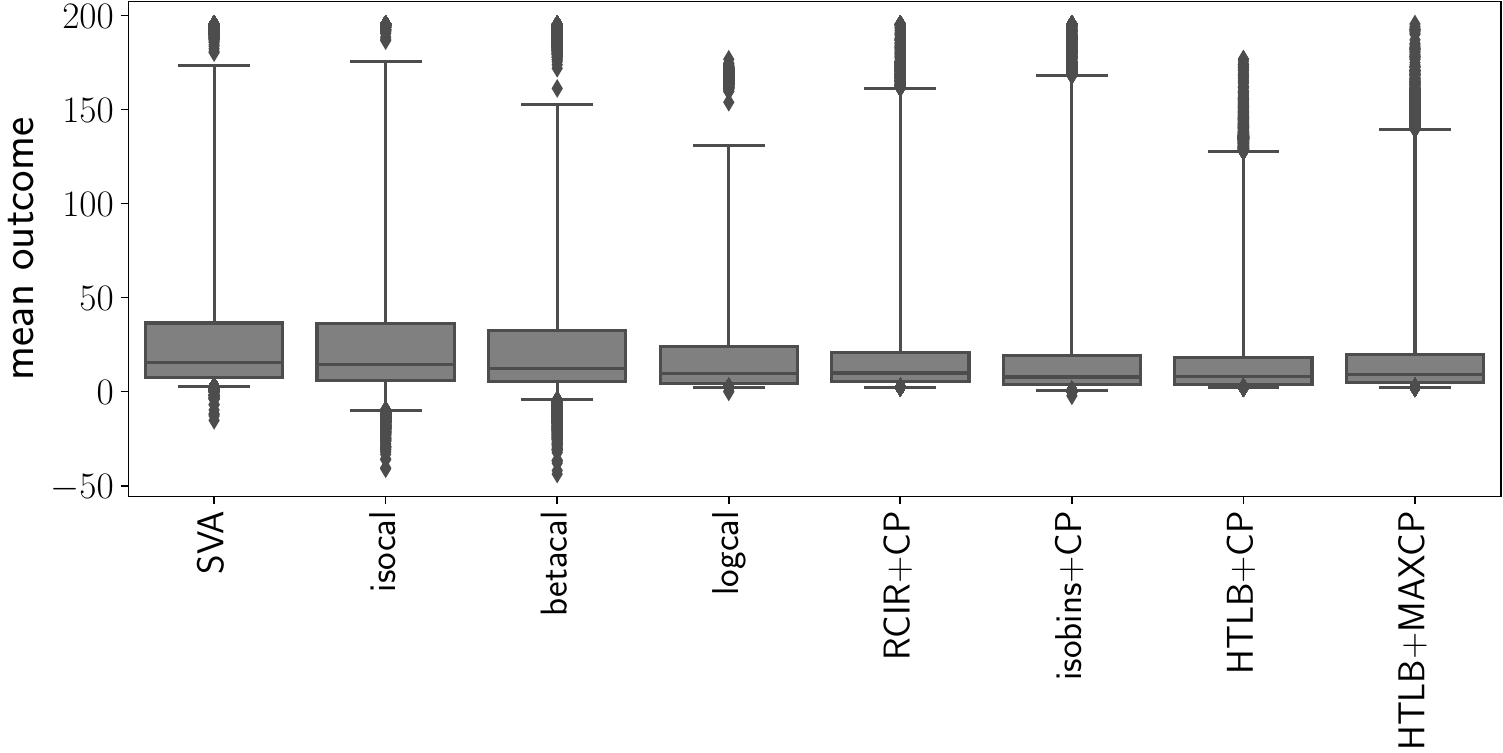}
  \caption{Boxplots for the mean outcome in each of the 50K learned maps for all methods (conservative versions).}
    \label{fig:mean_outcome}
\end{figure}

\begin{figure}[H]
  \centering
  \includegraphics[width=12cm]{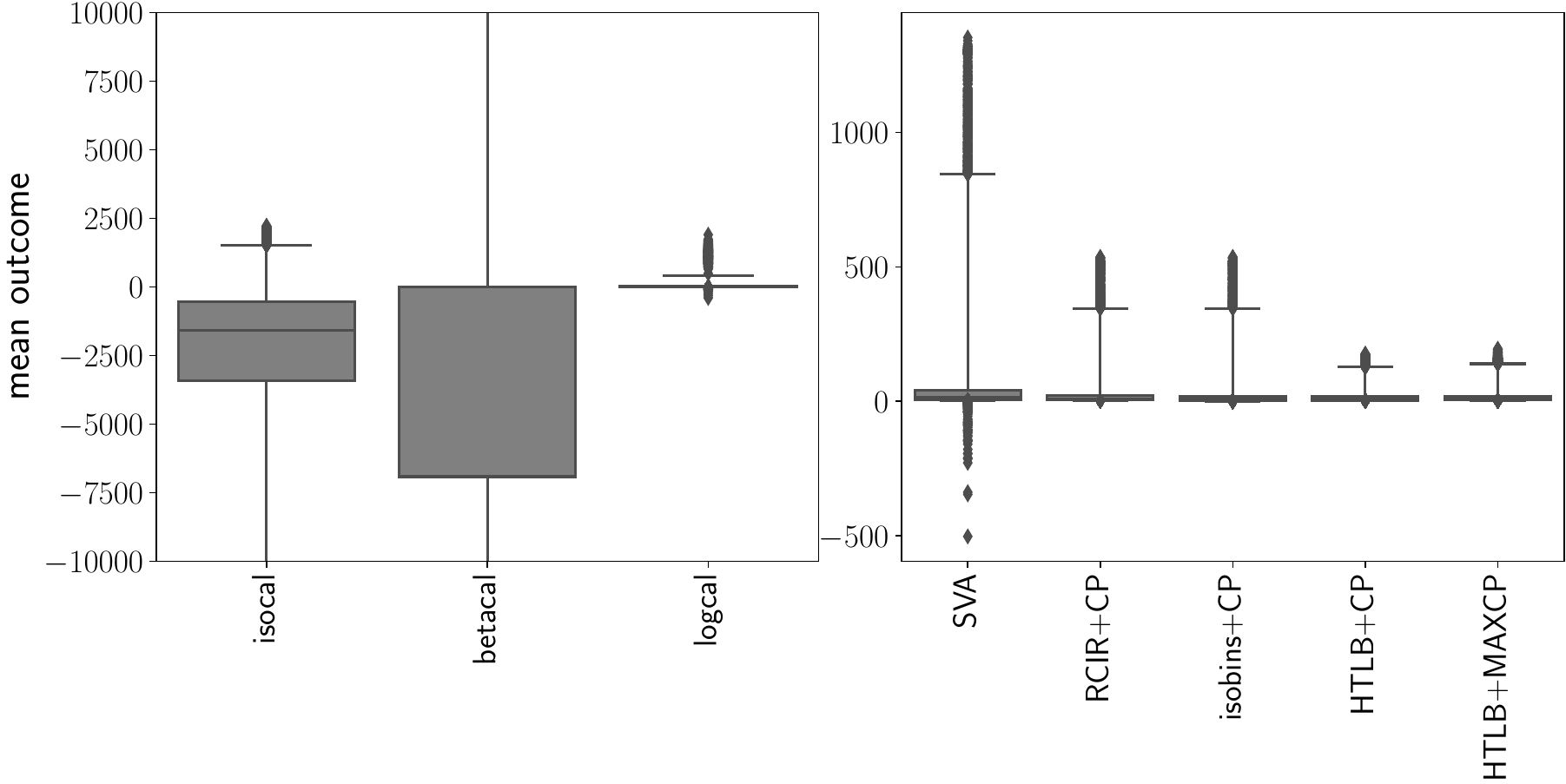}
  \caption{Boxplots for the mean outcome in each of the 50K learned maps for all methods (non-conservative versions, only monotonicity has been enforced to the non-monotonic methods).}
    \label{fig:mean_non_conservative}
\end{figure}

Next, in Figure \ref{fig:1perc_non_conservative}, we show the 1st percentile outcome results for the versions of methods where cutting post-processing wasn't used. Here, as said in the article, our methods have even bigger advantages as they are almost never having a negative 1st percentile outcome.

\begin{figure}[H]
  \centering
  \includegraphics[width=12cm]{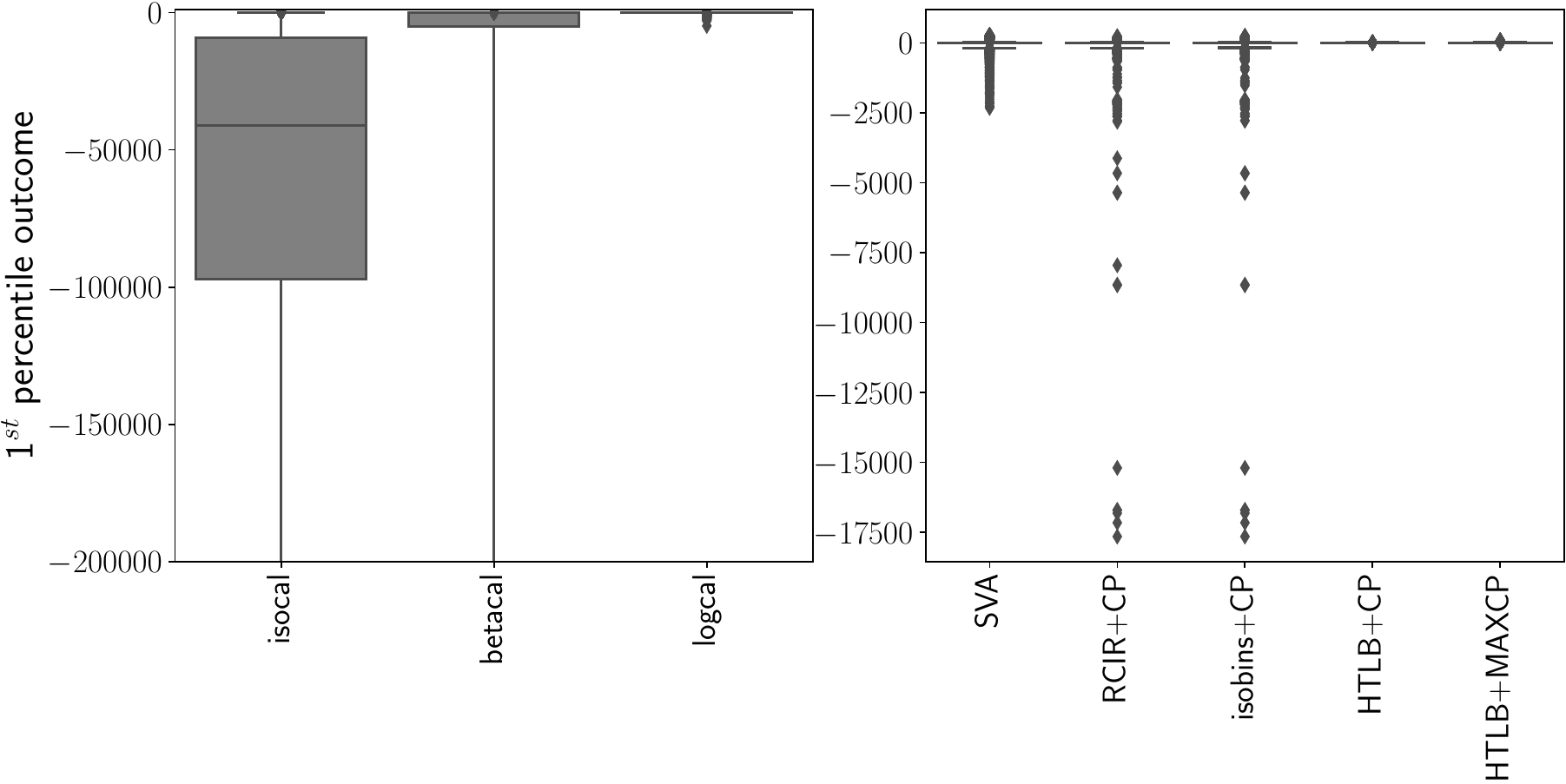}
  \caption{Boxplots for the 1st percentile outcome in each of the 50K learned maps for all methods (non-conservative versions, only monotonicity has been enforced to the non-monotonic methods)..}
    \label{fig:1perc_non_conservative}
\end{figure}

The same type of non-conservative results for the violation percentage can be seen in Figure \ref{fig:viol_non_conservative}. This is quite similar to the conservative version, hinting that the cutting didn't reduce all overconfidence, but decreased it enough so the example scenario expected outcome results were highly influenced by it.

\begin{figure}[H]
  \centering
  \includegraphics[width=12cm]{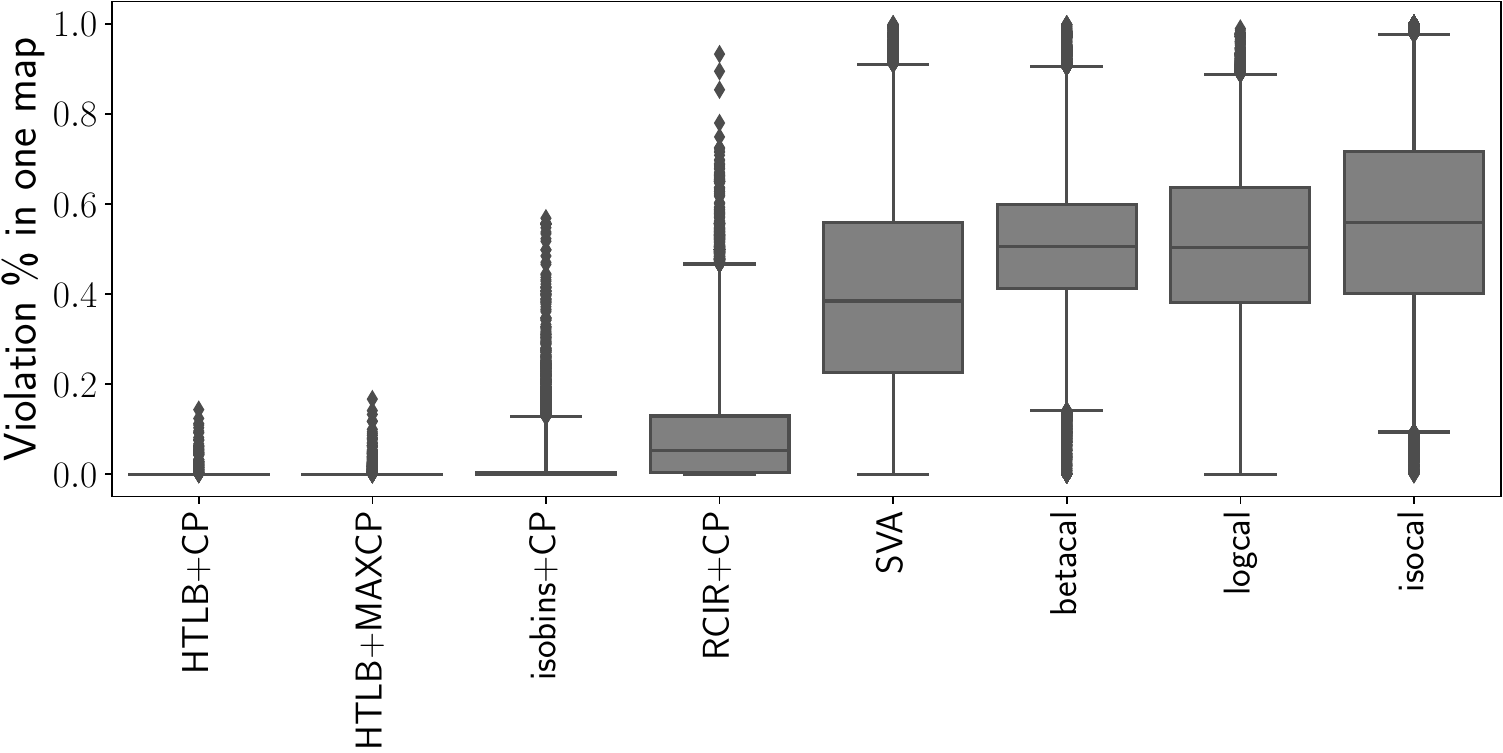}
  \caption{Boxplots for violation percentage in each of the 50K learned maps for all methods (non-conservative versions, only monotonicity has been enforced to the non-monotonic methods).}
    \label{fig:viol_non_conservative}
\end{figure}

\section{EXPERIMENT DETAILS}\label{exp_details}

Here are some remarks about the experiment setup details and reasoning behind choosing certain parameters:
\begin{enumerate}
    \item 100 true calibration maps are generated with $n=10000$ and values increasing between 0.9 and 1.0.
    \item 500 calibration sets are generated for each true calibration map.
    \item When evaluating, the first 2000 points are not included in the metrics, as there are no estimations for those points with HTLB+CP and HTLB+MAXCP methods. This is most likely okay in practice, since it is unlikely that the first points are of interest to us as they usually have lower calibrated probabilities. In practice, one could estimate those probabilities with another method, if some estimations are needed. But still, lower bounds will not be useful for probabilities $< 0.5$. These details can be decided on for each practical problem separately.
\end{enumerate}

\section{CODE}\label{code}



The code is accessible from Github:
\begin{verbatim}
    https://github.com/mlallikivi/cautious-calibration
\end{verbatim}
\endgroup

\end{document}